\documentclass{article}

\PassOptionsToPackage{sort,numbers}{natbib}
\usepackage[preprint]{neurips_2021}

\usepackage[utf8]{inputenc} % allow utf-8 input
\usepackage[T1]{fontenc}    % use 8-bit T1 fonts
\usepackage{xcolor}\definecolor{darkblue}{rgb}{0,0,0.75}
\usepackage[colorlinks=true, citecolor=darkblue, urlcolor=darkblue,linkcolor=darkblue]{hyperref}       % hyperlinks
\usepackage{url}            % simple URL typesetting
\usepackage{booktabs}       % professional-quality tables
\usepackage{amsfonts}       % blackboard math symbols
\usepackage{nicefrac}       % compact symbols for 1/2, etc.
\usepackage{microtype}      % microtypography
\usepackage{amsthm}
\usepackage{amsmath}
\usepackage{amssymb}
\usepackage{xfrac}
\usepackage{wrapfig}
\usepackage[caption=false]{subfig}
\usepackage{dsfont}

\usepackage[export]{adjustbox}
    
\usepackage{tikz}
\usepackage{verbatim}
\usetikzlibrary{decorations.pathreplacing}
\usepackage{soul} % use this (many fancier options)

\newcommand{\RR}{\mathbb{R} }

\newcommand{\Exp}[2]{\mathop{{}\mathbb{E}_{#1}} \Big[ #2 \Big] }

\newcommand{\Acal}{\mathcal{A}}

\newcommand{\Dcal}{\mathcal{D}}
\newcommand{\Ecal}{\mathcal{E}}

\newcommand{\w}{\mathbf{w}}
\newcommand{\uu}{\mathbf{u}}
\newcommand{\vv}{\mathbf{v}}

\newcommand{\loss}{\ell}
\newcommand*{\vvec}[1]{#1}

\newcommand{\modelspace}{\mathcal{H}}

\newcommand{\Xcal}{\mathcal{X}}
\newcommand{\Ycal}{\mathcal{Y}}

\newcommand{\defemph}[1]{\emph{#1}}

\newcommand*\samethanks[1][\value{footnote}]{\footnotemark[#1]}

\newcommand{\experimentURL}{https://github.com/kampmichael/relativeFlatnessGeneralization}

\newtheorem{theorem}{Theorem}

\newtheorem{proposition}[theorem]{Proposition}
\newtheorem{definition}[theorem]{Definition}
\newtheorem{corollary}[theorem]{Corollary}

\usepackage{xr}
\makeatletter
\newcommand*{\addFileDependency}[1]{% argument=file name and extension
  \typeout{(#1)}
  \@addtofilelist{#1}
  \IfFileExists{#1}{}{\typeout{No file #1.}}
}
\makeatother

\title{Relative Flatness and Generalization%\\in the Interpolation Regime
}

\author{%
  Henning Petzka\thanks{equal contribution} \\
  Lund University, Sweden\\
  \texttt{henning.petzka@math.lth.se} \\
  \And
  Michael Kamp\samethanks \\
  CISPA Helmholtz Center for Information Security,\\ Germany
  and Monash University, Australia\\
  \texttt{michael.kamp@monash.edu} \\
  \And
  Linara Adilova\\
  Ruhr University Bochum, Germany\\
  and Fraunhofer IAIS
  \And
  Cristian Sminchisescu\\
  Lund University, Sweden\\
  and Google Research, Switzerland\\
  \And
  Mario Boley\\
  Monash University, Australia\\
}

\hypersetup{final}

\begin{document}

\maketitle

\begin{abstract}
    Flatness of the loss curve is conjectured to be connected to the generalization ability of machine learning models, in particular neural networks. While it has been empirically observed that flatness measures
    consistently correlate strongly with generalization, it is still an open theoretical problem
    %based on the loss Hessian 
    %have a higher correlation with generalization than alternatives like weight norms, margin-, and optimization-based measures.
    why and under which circumstances flatness is connected to generalization, in particular in light of reparameterizations that change certain flatness measures but leave generalization unchanged.
    %However, it is an open theoretical problem why and under which circumstances flatness is connected to generalization, in particular in light of reparameterizations that change certain flatness measures but leave generalization unchanged.
    %However, it is an open theoretical problem why and under which circumstances flatness is connected to generalization, in particular in light of the reparameterization problem.
    We investigate the connection between flatness and generalization by relating it to the interpolation from representative data, deriving notions of representativeness, and feature robustness. The notions allow us to rigorously connect flatness and generalization and to identify conditions under which the connection holds. Moreover, they give rise to a novel, but natural relative flatness measure that correlates strongly with generalization, simplifies to ridge regression for ordinary least squares, and solves the reparameterization issue.
    %for ReLU neural networks.
\end{abstract}

% Flatness of the loss curve is conjectured to be connected to the generalization ability of machine learning models, in particular neural networks. Indeed, it has been empirically observed that flatness measures
%     consistently correlate strongly with generalization.
%     However, it is an open theoretical problem why and under which circumstances flatness is connected to generalization, in particular in light of reparameterizations that change certain flatness measures but leave generalization unchanged.
%     This paper investigates this connection by relating it to the interpolation from representative data, deriving notions of representativeness and feature robustness. This allows to rigorously connect flatness and generalization and to identify conditions under which the connection holds. Moreover, these notions give rise to a novel, but natural relative flatness measure that correlates strongly with generalization, simplifies to ridge regression for ordinary least squares, and solves the reparameterization issue.

\section{Introduction}
Flatness of the loss curve has been identified as a potential predictor for the generalization abilities of machine learning models~\citep{flatMinima, dzigugaite, simplifyingByFlat}. In particular for neural networks, it has been repeatedly observed that generalization performance correlates with measures of flatness, i.e., measures that quantify the change in loss under perturbations of the model parameters~\citep{entropySGD,  keskarLarge, foret2021sharpnessaware, zheng2020regularizing, sun2020exploring,adversarial_weightPerturbation, fisherRao, yao2018hessian}. In fact, \citet{jiang2020fantastic} perform a large-scale empirical study and find that flatness-based measures have a higher correlation with generalization than alternatives like weight norms, margin-, and optimization-based measures.
It is an open problem why and under which circumstances this correlation holds, in particular in the light of negative results on reparametrizations of ReLU neural networks~\citep{dinhSharp}: these reparameterizations change traditional measures of flatness, yet leave the model function and its generalization unchanged, making these measures unreliable. 
%This paper 
We present a novel and rigorous approach to understanding the connection between flatness and generalization by relating it to the interpolation from representative samples. 
Using this theory we, for the first time, identify conditions under which flatness explains generalization. At the same time, we derive a %natural 
measure of \textit{relative flatness} that simplifies to ridge/Tikhonov regularization for ordinary least squares~\citep{tikhonov}, and  resolves the reparametrization issue for ReLU  networks~\cite{dinhSharp} by appropriately taking the norm of parameters into account as suggested by \citet{neyshabur2015norm}.

Formally, we connect flatness of the loss surface to the \defemph{generalization gap} $\Ecal_{gen}(f,S)= \Ecal(f) - \Ecal_{emp}(f,S)$ of a \defemph{model} $f:\Xcal\rightarrow\Ycal$ from a \defemph{model class} $\modelspace$ with respect to a twice differentiable \defemph{loss function} $\loss:\Ycal\times\Ycal\rightarrow\RR_+$ and a finite sample set $S\subseteq\Xcal\times \Ycal$, where 
\[
 \Ecal(f) = \Exp{(x,y)\sim \Dcal}{\loss(f(x),y)} \ \text{ and }\  \Ecal_{emp}(f,S) = \frac{1}{|S|} \sum_{(x,y)\in S} \loss( f(x),y)\enspace .
\]
That is, $\Ecal_{gen}(f,S)$ is the difference between the \defemph{risk} $\Ecal(f)$  and the \defemph{empirical risk} $\Ecal_{emp}(f,S)$ of $f$ on a finite sample set $S$ %\subset\Xcal\times\Ycal$ %of size $N\in\NN$ 
drawn iid.\ according to a \defemph{data distribution} $\Dcal$ on $\Xcal\times\Ycal$. %
To connect flatness to generalization, we start by decomposing the generalisation gap into two terms, a \emph{representativeness} term that quantifies how well a distribution $\Dcal$ can be approximated using distributions with local support around sample points and a \textit{feature robustness} term describing how small changes of feature values affect the model's loss.
Here, feature value refers to the implicitly represented features by the model, i.e., we consider models that can be expressed as $f(x)=\psi(\w,\phi(x))=g(\w\phi(x))$ with a feature extractor $\phi$ and a model $\psi$ (which includes linear and kernel models, as well as most neural networks, see Fig.~\ref{fig:featureDecomposition}).
%%%
With this decomposition, we measure the generalization ability of a particular model by how well its interpolation between samples in feature space fits the underlying data distribution.
%%%%%%%%%%%%%%%%%%%%%%%%%%%
We then connect feature robustness (a property of the feature space) to flatness (a property of the parameter space) using the following key identity:
Multiplicative perturbations in feature space by arbitrary matrices $A\in\RR^{m\times m}$ correspond to  perturbations in parameter space, i.e.,
\begin{equation}\label{eq:keyEquation}
 \psi(\w, \phi(x)+A\phi(x))=g(\w (\phi(x)+ A\phi(x))) = g((\w+ \w A )\phi(x))=\psi (\w+ \w A,\phi(x))\enspace .
\end{equation} 
\begin{figure}[t]
    \centering
    \begin{minipage}{0.39\textwidth}
        \centering
        \def\layersep{2.5cm}

\definecolor{input}{RGB}{56,54,80}
\definecolor{hidden}{RGB}{0,193,189}
\definecolor{output}{RGB}{179,132,184}

\begin{tikzpicture}[shorten >=1pt,->,draw=black!50, node distance=\layersep, scale=0.585, transform shape]%scale=0.7, transform shape]
    \tikzstyle{every pin edge}=[<-,shorten <=1pt]
    \tikzstyle{neuron}=[circle,fill=black!25,minimum size=17pt,inner sep=0pt]
    \tikzstyle{input neuron}=[neuron, fill=input];
    \tikzstyle{output neuron}=[neuron, fill=output];
    \tikzstyle{hidden neuron}=[neuron, fill=hidden];
    \tikzstyle{hidden neuron 2}=[neuron, fill=hidden];
    \tikzstyle{annot} = [text width=4em, text centered]

    % Draw the input layer nodes
    \foreach \name / \y in {1,...,4}
    % This is the same as writing \foreach \name / \y in {1/1,2/2,3/3,4/4}
        \node[input neuron, pin=left:] (I-\name) at (0,-\y) {};

    % Draw the hidden layer nodes
    \foreach \name / \y in {1,...,5}
        \path[yshift=0.5cm]
            node[hidden neuron] (H1-\name) at (\layersep,-\y cm) {};
        
    % Connect every node in the input layer with every node in the
    % hidden layer.
    \foreach \source in {1,...,4}
    	\foreach \dest in {1,...,5}
    		\path (I-\source) edge (H1-\dest);    
            
    % Draw the hidden layer nodes
    \foreach \name / \y in {1,...,4}
    	\path[yshift=0.0cm]
	    	node[hidden neuron 2] (H2-\name) at (2.2*\layersep,-\y cm) {};

	% Connect every node in the input layer with every node in the
	% hidden layer.
	\foreach \source in {1,...,5}
		\foreach \dest in {1,...,4}
			\path[draw=black!50] (H1-\source) edge (H2-\dest);  

    % Draw the output layer node
    \node[output neuron,pin={[pin edge={->}]right:}, right of=H2-3] (O) at (2*\layersep,-2.5 cm) {};
    %\node[output neuron,pin={[pin edge={->}]right: $f = \psi\circ\phi$}, right of=H2-3] (O) at (2*\layersep,-2.5 cm) {};

    % Connect every node in the hidden layer with the output layer
    \foreach \source in {1,...,4}
        \path (H2-\source) edge (O);

    % Annotate the layers
    \node[annot,node distance=1cm] (hl) at (4,0.6 cm) {hidden layers};
    \node[annot] at (0,0.6 cm) {input layer};
    \node[annot] at (7.5,0.6 cm) {output layer};

    %draw the curly brackets
    \draw [-,
    thick,
    decoration={
    	brace,
    	mirror,
    	raise=0.3cm,
    	amplitude=10pt
    },
    decorate
    ] (-0.5,-4.3) -- (3,-4.3) 
    node [pos=0.5,anchor=north,yshift=-1.3cm] {$\phi$}; 
    
    %draw the curly brackets
    \draw [-,
    thick,
    decoration={
    	brace,
    	mirror,
    	raise=0.3cm,
    	amplitude=10pt
    },
    decorate
    ] (3.5,-4.3) -- (8,-4.3) 
    node [pos=0.5,anchor=north,yshift=-1.3cm] {$\psi$}; 
\end{tikzpicture}
        \caption{Decomposition of $f=\psi\circ\phi$ into a feature extractor $\phi$ and a model $\psi$ for neural networks.}
        \label{fig:featureDecomposition}
    \end{minipage}
        \hspace{0.1cm}
    \begin{minipage}{.59\textwidth}
        \centering
	    \includegraphics[width=8cm]{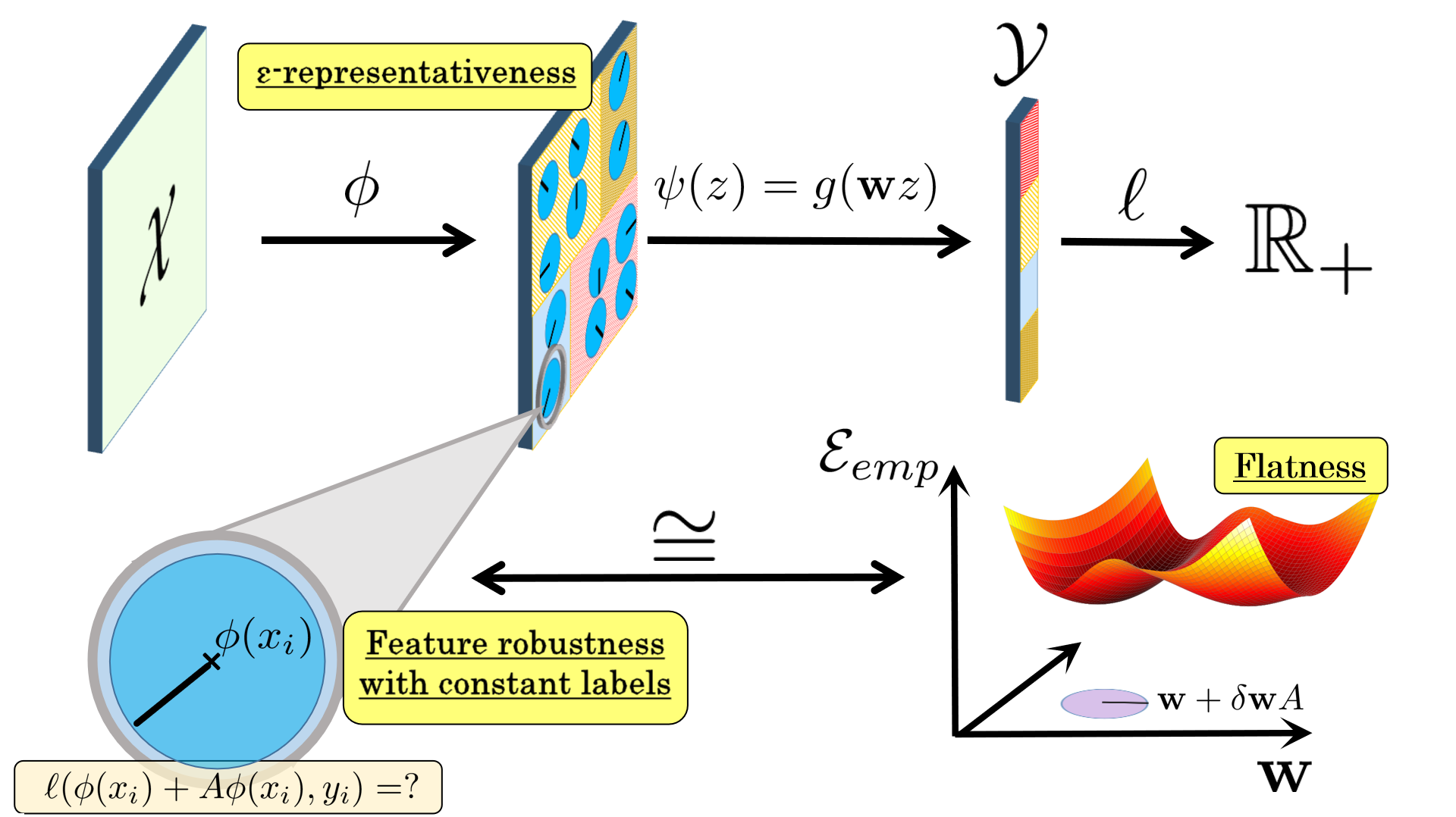}
	    \caption{Overview: We theoretically connect a notion of representative data with a notion of feature robustness and a novel measure of flatness of the loss surface.}
	    \label{fig:overview}
    \end{minipage}
\end{figure}
%
% To express this, we propose a natural measure of \textit{representativeness} 
% %of a sample set $S\sim\Dcal$ 
% that captures how well a distribution $\Dcal$ can be approximated using distributions with local support around sample points. Using this notion, we can now decompose the generalization gap into representativeness and \textit{feature robustness}, a term describing how small changes of feature values affect the model's loss. 
%
Using this key equation, we show that feature robustness is approximated by a novel, but natural, loss Hessian-based \emph{relative flatness} measure under the assumption that the distribution can be approximated by 
%a distribution with 
locally constant labels. Under this assumption and if the data is representative, then flatness is the main predictor of generalization (see Fig.~\ref{fig:overview} for an illustration).

This offers an explanation for the 
%noisy but steady 
correlation of flatness with generalization on many real-world data distributions for image classification~\citep{jiang2020fantastic,fisherRao, exploringGen}, where the assumption of locally constant labels is reasonable (the definition of adversarial examples~\citep{intriguingProperties} even hinges on this assumption).
This dependence on locally constant labels has not been uncovered by previous theoretical analysis~\citep{tsuzuku, exploringGen}. Moreover, we show that the resulting relative flatness measure is invariant to linear reparameterization and has a stronger correlation with generalization than other flatness measures~\citep{fisherRao, exploringGen,jiang2020fantastic}.
Other measures have been proposed that similarly %scale the loss Hessian with parameter norms to 
achieve invariance under reparameterizations~\citep{fisherRao, tsuzuku}, but the Fisher-Rao norm~\citep{fisherRao} is lacking a strong theoretical connection to generalization, our measure sustains a more natural form than normalized sharpness~\citep{tsuzuku} and for neural networks, it considers only a single layer, given by the decomposition of $f$ (including the possibility of choosing the input layer when $\phi=id_\Xcal$).
An extended comparison to related work is provided in Appdx.~\ref{sct:related_work}.

The limitations of our analysis are as follows. We assume
a noise-free setting where for each $x\in\Xcal$ there is a unique $y=y(x) \in \Ycal$ such that $P_{x,y\sim D}(y | x) = 1$, and this assumption is also extended to the feature space of the given model, i.e., we assume that $\phi(x)=\phi(x')$ implies $y(x) = y(x')$ for all $x, x' \in \Xcal$ and write $y(x)=y(\phi(x))$. 
Moreover, we assume that the marginal distribution $\Dcal_\Xcal$ 
%is absolute continuous with respect to the Lebesgue measure on $\Xcal$ so that the distribution $\Dcal$
is described by a density function $p_{\Dcal}(x)$, that $f(x)=\psi(\w,\phi(x))=g(\w\phi(x))$ is a local minimizer of the empirical risk on $S$, and that $g$, $\psi$, $\phi$ are twice differential.
%For neural networks, the decomposition of $f$ means that we choose one layer of the network for our analysis, which includes the possibility of considering the input layer when $\phi=id_\Xcal$. 
%
Quantifying the representativeness of a dataset precisely is challenging since the data distribution is unknown. Using results from density estimation, we derive a worst-case bound on representativeness for all data distributions that fulfill mild regularity assumptions in feature space $\phi(\Xcal)$, i.e., a smooth density function $p_{\phi(\Dcal)}$ such that $\int_{z\in \phi(\Xcal)}  \left | \nabla^2 \left( p_{\phi(\Dcal)}(z)||z||^2\right ) \right | dz $ and $\int_{z\in\phi(\Xcal)} {p_{\phi(\Dcal)}(z)}/{||z||^m}\ dz$ are well-defined and finite.
% using results from density estimation. 
%to quantify representativeness, we derive a worst-case bound under mild regularity assumptions on the data distribution using standard results from density estimation. 
This yields a generalization bound incorporating flatness. 
In contrast to the common bounds of statistical learning theory, the bound depends on the feature dimension. The dimension-dependence is a result of the interpolation approach (applying density estimation uniformly over all distributions that satisfy the mild regularity assumptions). The bound is consistent with the no-free-lunch theorem and the convergence rate derived by~\citet{belkinPerfectFit} for a model based on interpolations. 
In practical settings, representativeness can be expected to be much smaller than the worst-case bound, which we demonstrate by a synthetic example in Sec.~\ref{sec:experiments}. 
Generally, it is a bound that remains meaningful in the interpolation regime~\citep{bartlett2020benign, belkin2019reconciling, indcutiveBias}, where traditional measures of generalization based on the empirical risk and model class complexity are uninformative~\citep{understandingRethinking, nagarajan2019uniform}.  
%which can be alleviated by considering representativeness and flatness not in the input space, but rather in a feature space (e.g., the output of some layer in a neural network, see Fig.~\ref{}). 

\textbf{Contribution.} In summary, this paper 
%for the first time 
rigorously connects flatness of the loss surface to generalization and shows that this connection requires feature representations such that labels are (approximately) locally constant, which is also validated in a synthetic experiment (Sec.~\ref{sec:experiments}). The empirical evaluation shows that this flatness and an approximation to representativeness can tightly bound the generalization gap. % and that the novel flatness measure strongly correlates with generalization.
Our contributions are:
% \begin{itemize}
%     \item[(i)] the rigorous connection of flatness and generalization,
%     \item[(ii)] novel notions of representativeness and feature robustness that capture the extent to which a model's interpolation between samples fits the data distribution, and
%     \item[(iii)] a novel flatness measure that is layer- and neuron-wise reparameterization invariant, reduces to ridge regression for ordinary least squares, and outperforms state-of-the-art flatness measures on CIFAR10.
% \end{itemize}
(i) the rigorous connection of flatness and generalization;
(ii) novel notions of representativeness and feature robustness that capture the extent to which a model's interpolation between samples fits the data distribution; and 
(iii) a novel flatness measure that is layer- and neuron-wise reparameterization invariant, reduces to ridge regression for ordinary least squares, and outperforms state-of-the-art flatness measures on CIFAR10.

\section{Representativeness}\label{sct:representativeness}
In this section, we formalize when a sample set $S$ is representative for a data distribution $\Dcal$.

\textbf{Partitioning the input space.} We choose a partition $\{V_i\ |i=1,\ldots,|S|\}$ of $\Xcal$ such that each element of this partition $V_i$ contains exactly one of the samples $x_i$ from $S$. The distribution can then be described by a set of densities $p_i(x)=\frac{1}{\alpha_i}\cdot p_\Dcal(x) \cdot \mathbf{1}_{V_i}(x)$ with support contained in $V_i$ (where $\mathbf{1}_{V_i}(x)=1$ if $x\in V_i$ and $0$ otherwise) and with normalizing factor $\alpha_i=\int_{V_i} p_\Dcal(x)  dx$. Then %we can write the risk as
%\[\Ecal(f) = \int p_\Dcal(x) \loss(f(x),y(x)) dx = \sum_{i=1}^n \alpha_i \int p_i(x) \loss(f(x),y(x)) dx=  \sum_{i=1}^n \alpha_i \mathbb{E}_{x\sim p_i}[ \loss(f(x,y(x))]. \]
the risk decomposes as $\Ecal(f) = \sum_{i=1}^{|S|} \alpha_i\cdot  \mathbb{E}_{x\sim p_i}[ \loss(f(x),y(x))].$
Since $x_i\in V_i$ for each $i$, we can change variables and consider density functions $\lambda_i^*(\xi)=p_i(x_i+\xi)$ with support in a neighborhood around the origin of $\Xcal$. The risk then decomposes as
%\[\Ecal(f)= \int p_\Dcal(x) \loss(f(x),y(x)) dx =  \sum_{i=1}^n \alpha_i \mathbb{E}_{x\sim \lambda_i}[ \loss(f(x_i+x),y(x_i+x))]. \]
\begin{equation}\label{eq:partition}
\Ecal(f)=  \sum_{i=1}^{|S|} \alpha_i \cdot \mathbb{E}_{\xi\sim \lambda_i^*}[ \loss(f(x_i+\xi),y(x_i+\xi))] %\ \text{ , where }\alpha_i^*=\int_{\Xcal} \lambda_i^*(x)  dx
\enspace .\end{equation}
Starting from this identity, we formalize an approximation to the risk: In a practical setting, the distribution $p_\Dcal$ is unknown and hence, in the decomposition \eqref{eq:partition}, we have unknown densities $\lambda_i^*$ and unknown normalization factors $\alpha_i$. We assume that each neighborhood contributes equally to the loss, i.e., we approximate each $\alpha_i$ with $\frac 1{|S|}$. Then, given a sample set $S$ and an $|S|$-tuple $\Lambda = (\lambda_i)_{1\leq i\leq |S|}$ of ``local'' probability density functions on $\Xcal$ with support  ${supp}(\lambda_i)$ in a neighborhood around the origin $0_\Xcal$, we call the pair $(S,\Lambda)$ \defemph{ $\epsilon$-representative for $\Dcal$} with respect to a model $f$ and loss $\loss$ if $|\Ecal_{Rep}(f,S,\Lambda) |\leq \epsilon$, where 
%\begin{equation}\label{eq:representativeness}
%     Rep(f,S,\Lambda)= \Ecal(f)   -\sum_{i=1}^{|S|} \frac 1{|S|} \cdot \mathbb{E}_{\xi\sim \lambda_i}[ \loss(f(x_i+\xi),y(x_i+\xi))].
%\end{equation}
\begin{equation}\label{eq:interpolating}
     \Ecal_{Rep}(f,S,\Lambda)= \Ecal(f)   - \sum_{i=1}^{|S|} \frac 1{|S|}\cdot \mathbb{E}_{\xi \sim \lambda_i}\left[\loss( f(x_i+ \xi),y(x_i+\xi))\right]\enspace .
\end{equation}

If the partitions $V_i$ and the distributions $\lambda_i$ are all chosen optimal so that the approximation $\alpha_i=\frac 1{|S|}$ is exact and $\lambda_i=\lambda_i^*$, then $\Ecal_{Rep}(f,S,\Lambda)=0$ by \eqref{eq:partition}. If the support of each $\lambda_i$ is decreased to the origin so that $\lambda_i=\delta_{0}$ is a Dirac delta function, then $\Ecal_{Rep}(f,S,\Lambda)=\Ecal_{gen}(f,S)$ equals the generalization gap.
%, and $\epsilon$-representativeness is equivalent to classical representativeness~\citep[][Df. 4.1]{learningTheoryBook}. 
For density functions with an intermediate support, the generalization gap can be decomposed into representativeness and the expected deviation of the loss around the sample points:
\[\Ecal_{gen}(f,S) = \Ecal_{Rep}(f,S,\Lambda)+  \sum_{i=1}^{|S|} \frac 1{|S|} \cdot \mathbb{E}_{\xi\sim \lambda_i}[ \loss(f(x_i+\xi),y(x_i+\xi))-\loss(f(x_i),y_i)]\]
The main idea of our approach to understand generalization is to use this equality and to control both representativeness and expected loss deviations for a suitable $|S|$-tuple of distributions $\Lambda$.

\textbf{From input to feature space.} An interesting aspect of $\epsilon$-representativeness is that it can be considered in a feature space instead of the input space. For a model $f=(\psi\circ \phi) :\Xcal\rightarrow \Ycal$, we can apply
%consider $\phi$ as a feature extractor and $\psi$ as the model and apply 
our notion to the feature space $\phi(\Xcal)$ (see Fig.~\ref{fig:featureDecomposition} for an illustration). This leads to the notion of $\epsilon$-representativeness in feature space defined for an $|S|$-tuple $\Lambda^\phi =(\lambda^\phi_i)_{1\leq i\leq |S|}$ of densities on $\phi(\Xcal)$ by replacing $x_i$ with $\phi(x_i)$ in \eqref{eq:interpolating}, which we denote by  $\Ecal_{Rep}^\phi(f,S,\Lambda^\phi)$. By measuring representativeness in a feature space, this becomes a notion of both data and feature representation. In particular, it assumes that a target output function $y(\phi(x))$ also exists for the feature space.
%by  
%\begin{equation}\label{eq:representativenessFeature}
%     Int^\phi(f,S,\Lambda^\phi)= \Ecal(f)   -\frac 1n \sum_{(x_i,y_i)\in S} \mathbb{E}_{\xi \sim \lambda_i^\phi}\left[\loss( \psi(\phi(x_i)+ \xi),y_i)\right]
%\end{equation}
We can then decompose the generalization gap  $\Ecal_{gen}(f)$ of $f=(\psi\circ \phi)$ into
\[ \Ecal_{Rep}^\phi(f,S,\Lambda^\phi) + \left(\frac 1{|S|} \sum_{i=1}^{|S|} \mathbb{E}_{\xi \sim \lambda_i^\phi}\left[\loss( \psi(\phi(x_i)+ \xi),y(\phi(x_i)+\xi))- \loss(f(x_i),y_i) \right] \right ) \]
%\begin{wrapfigure}{h}{0.5\textwidth}
%	\centering
%	\input{images/functionIllustration}
%	\caption{Illustration of the decomposition into a feature extractor and a model, i.e., $f=\psi\circ\phi$, for neural networks.}
%	\label{fig:functionIllustration}
%\end{wrapfigure}
The second term is determined by how the loss changes under small perturbations in the feature space for the samples in $S$.  As before,  for $\lambda_i=\delta_0$  the term in the bracket vanishes and $\Ecal_{Rep}^\phi(f,S,\Lambda^\phi)=\Ecal_{gen}$.  But the decomposition becomes more interesting for distributions with support of nonzero measure around the origin. %Having trained a neural network to a local minimizer of the empirical error, the term in the bracket is positive for neighborhoods with support of positive measure. 
 If the true distribution can be interpolated efficiently in feature space from the samples in $S$ with suitable $\lambda_i^\phi$ so that  $\Ecal_{Rep}^\phi(f,S,\Lambda^\phi)\approx 0$, then the term in the bracket approximately equals the generalization gap and the generalization gap can be estimated from local properties in feature space  around sample points. 
 %In the next section, we will define a novel notion of feature robustness of a model that is able to exploit the above connection between $\epsilon$-representativeness and the generalization gap for natural distributions $\Lambda$. 
 %. This notion of feature robustness naturally describes properties of the loss surface as a function of the weight parameters, 
%which finally allows us to connect generalization to a measure of flatness of the loss surface.

\section{Feature Robustness}\label{sec:featureRobustness}
Having decomposed the generalisation gap into a representativness and a second term of loss deviation, we now develop a novel notion of feature robustness that is able to bound the second term for specific families of distributions $\Lambda$ using key equation~\eqref{eq:keyEquation}. 
% To connect feature robustness (a property of the feature space) to flatness (a property of the parameter space) we make use of the following key identity: Multiplicative perturbations in feature space by arbitrary matrices $A\in\RR^{m\times m}$ correspond to  perturbations in parameter space, i.e.,
% \begin{equation}\label{eq:keyEquation}
%     \psi(\w, \phi(x)+A\phi(x))=g(\w (\phi(x)+ A\phi(x))) = g((\w+ \w A )\phi(x))=\psi (\w+ \w A,\phi(x))\enspace .
% \end{equation} 
Our definition of \defemph{feature robustness} for a model $f=(\psi\circ \phi) :\Xcal\rightarrow \Ycal$ depends on a small number $\delta>0$, a sample set $S$ and a \defemph{feature selection} defined by a matrix $A\in \RR^{m\times m}$ of \defemph{operator norm} $||A||\leq 1$.  With feature perturbations $\phi_A(x)=(I+A)\phi(x)$ and
\begin{equation}\label{eq:defRobustness}
\begin{split}
\Ecal_{\mathcal{F}}^\phi(f, S, A):= \frac 1{|S|} \sum_{i=1}^{|S|} &\Big[ \loss(\psi( \phi_A(x_i)),\ y[\phi_A(x_i)]) -\loss(f(x_i),y_i)\Big],
%\mathcal{F} (f, S, A):= \frac 1{|S|} \sum_{(x,y)\in S} &\Big[ \loss(\psi[ \phi(x) + A\phi(x)],\ y[\phi(x)+A\phi(x)]) -\loss(f(x),y(x))\Big],
%\mathcal{F} (f, x, A):= \loss(\psi( \phi(x) + A\phi(x)),y(\phi(x) + A\phi(x)))) -\loss(f(x),y),
\end{split}
\end{equation}
%$$\sum_{i=1}^n \beta_i \cdot \mathbb{E}_{A\in\Acal_i}\left [ \mathcal{F}(f,x_i,A) \right ]$$
%$$ \beta_i=  \int_{\{ Ax_i|\ A\in supp\mathcal{A}_i\}} p_\Dcal(x)  dx$$
the definition of feature robustness is given as follows.
\begin{definition}\label{def:featureRobustness}
Let $\loss:\Ycal\times \Ycal\rightarrow \RR_+$ denote a loss function, $\epsilon$ and $\delta$ two positive (small) real numbers, $S \subseteq \Xcal\times \Ycal$ a finite sample set, and $A\in \RR^{m\times m}$ a matrix. A model $f(x)=(\psi\circ \phi)(x)$ with $\phi(\Xcal)\subseteq\RR^m$ 
%which is a composition of functions $\phi:\Xcal \rightarrow \RR^m$ and $\psi:\RR^m\rightarrow \Ycal$, 
is called $\pmb{\left ( (\delta,S,A),\epsilon\right )}$\defemph{-feature robust}, if $\left |\Ecal_{\mathcal{F}}^\phi(f,S, \alpha A)\right |\leq \epsilon\ \textrm{ for all } 0\leq \alpha\leq \delta.$ More generally, for a probability distribution $\mathcal{A}$ on perturbation matrices in $\RR^m$, we define 
\[\Ecal_{\mathcal{F}}^\phi(f,S,\Acal) = \Exp{A\sim \mathcal{A}}{ \Ecal_{\mathcal{F}}^\phi(f,S, A)} \enspace ,%\textrm{ for all } 0\leq s\leq \delta.
\]
and call the model $\pmb{\left ( (\delta,S,\mathcal{A}),\epsilon\right )}$\defemph{-feature robust on average over $\mathcal{A}$}, if $\left |\Ecal_{\mathcal{F}}^\phi(f,S,\alpha\Acal)\right | \leq \epsilon$ for $0\leq \alpha\leq \delta$.
\end{definition}

Given a feature extractor $\phi$, feature robustness measures the performance of $\psi$ when feature values are perturbed (with constant feature extractor $\phi$). This local robustness at sample points differs from the robustness of~\citet{robustnessGeneralization} that requires a data-independent partitioning of the input space. The matrix $A$ in feature robustness determines which feature values shall be perturbed. For each sample, the perturbation is linear in the expression of the feature. Thereby, we only perturb features that are relevant for the output for a given sample and leave feature values unchanged that are not expressed.  For $\phi$ mapping into an intermediate layer of a neural network, traditionally, the activation values of a neuron are considered as feature values, which corresponds to a choice of $A$ as a projection matrix. 
%The feature value $\phi_j(x)$  defined by the $j$-th neuron in the feature space $\phi(x)\in\RR^m$ can be written as $\phi_j(x)=\langle\phi(x), e_j \rangle$, where $e_j$ denotes the $j$-th unit vector and $\langle \cdot,\cdot \rangle$ the scalar product in $\RR^m$. 
However, it was shown by \citet{intriguingProperties} that, for any other direction $v\in \RR^m, ||v||=1$, the values $\langle \phi(x), v \rangle$ obtained from the projection $\phi(x)$ onto $v$, can be likewise semantically interpreted as a feature. This motivates the consideration of general \defemph{feature matrices} $A$.

%Multiplication of $\phi(x)$ with a matrix $A$ corresponds to a weighted selection of $rank(A)$-many features in parallel (e.g., projection matrices on d-dimensional subspaces correspond to the selection of $d$ many features) justifying our terminology considering a matrix $A$ as a feature selection and calling such $A$ a \defemph{feature matrix} in the following. %The same way that, for a sample input $x$, non-activated neurons $\phi_j(x)=0$ are considered as non-expressed features, we call a selection of features defined by matrix $A$ as \defemph{non-expressed} whenever $A\phi(x)=0$.

\paragraph{Distributions on feature matrices induce distributions on the feature space %via multiplication.
}\label{sct:inducingDistributions} Feature robustness is  defined in terms of feature matrices (suitable for an application of~\eqref{eq:keyEquation} to connect perturbations of features with perturbations of weights), while the approach exploiting representative data from Section~\ref{sct:representativeness} considers distributions on feature vectors, cf.~\eqref{eq:interpolating}. To connect feature robustness to the notion of $\epsilon$-representativeness, we specify for any distribution $\Acal$ on matrices $A\in \RR^{m\times m}$ an $|S|$-tuple $\Lambda_{\mathcal{A}}=(\lambda_i)$ of probability density functions $\lambda_i$ on the feature space $\RR^m$ with support containing the origin. Multiplication of a feature matrix with a feature vector $\phi(x_i)$ defines a feature selection $A\phi(x_i)$, and for each $z\in \RR^m$ there is some feature matrix $A$ with $\phi(x_i)+z=\phi(x_i)+A\phi(x_i)$ (unless $\phi(x_i)=0$). Our choice for distributions $\lambda_i$ on $\RR^m$ are therefore distributions that are induced via multiplication of feature vectors $\phi(x_i)\in \RR^m$ with matrices $A\in \RR^{m\times m}$ sampled from a distribution on feature matrices $\Acal$ . Formally, we assume that a Borel measure $\mu_A$ is defined by a probability distribution $\mathcal{A}$ on matrices $\RR^{m\times m}$. We then define Borel measures $\mu_i$ on $\RR^m$ by $\mu_i(C)=\mu_A(\{ A\ |\ A\phi(x_i)\in C\})$ for Borel sets $C\subseteq \RR^m$. Then $\lambda_i$ is the probability density function defined by the Borel measure $\mu_i$. As a result, we have for each $i$ that
\[ \Exp{A \sim \mathcal{A}}{\loss(\psi(\phi_A(x_i)), y(\phi_A(x_i) ))} =  \Exp{z \sim \lambda_i}{\loss(\psi(\phi(x_i)+  z), y(\phi(x_i)+z))}\]
%\[ \Exp{A \sim \mathcal{A}}{\loss(\psi(\phi(x_i)+  A\phi(x_i)), y[\phi(x_i)+  A\phi(x_i)])} =  \Exp{z \sim \lambda_i}{\loss(\psi(\phi(x_i)+  z), y[\phi(x_i)+z])}\]

\paragraph{Feature robustness and generalization.}
With this construction and a distribution $\Lambda_{\mathcal{A}}$ on the feature space induced by a distribution $\mathcal{A}$ on feature matrices, we have that
\begin{equation}\label{eq:splitError}
\Ecal(f) = \Ecal_{emp}(f,S) + \Ecal_{Rep}^\phi(f,S,\Lambda_{\mathcal{A}}) + \Ecal_{\mathcal{F}}^\phi(f,S, \Acal)
\end{equation}
Here, $\mathcal{A}$ can be any distribution on feature matrices, which can be chosen suitably to control how well the corresponding mixture of local distributions approximates the true distribution. The third term then measures how robust the model is in expectation over feature changes for $A\sim\mathcal{A}$.   In particular, if $\Ecal_{Rep}^\phi(f,S,\Lambda_{\mathcal{A}})\approx 0$, then $\Ecal_{gen}(f,S)\approx \Ecal_{\mathcal{F}}^\phi(f,S, \Acal)$ and the generalization gap is determined by feature robustness. We end this section by illustrating how distributions on feature matrices induce natural distributions on the feature space. The example will serve in Sec.~\ref{sec:FRandGen} to deduce a bound on $\Ecal_{Rep}^\phi(f,S,\Lambda_{\mathcal{A}})$ from kernel density estimation.

\paragraph{Example: Truncated isotropic normal distributions}\hspace{-0.2cm}are induced by a suitable distribution on feature matrices.
%The paper considers an $|S|$-tuple of densities $\Lambda_{\delta,\mathcal{A}}$ induced by  a distribution on feature matrices are truncated scaled normal distributions with support around the training samples. 
We consider probability distributions $\mathcal{K}_{\delta ||\phi(x_i)||}$
%\begin{equation}\label{eq:kernelExample}
%\lambda_i=\mathcal{K}_{\delta||\phi(x_i)||}
%\end{equation} 
on feature vectors $z\in\RR^m$ in the feature space defined by densities $k_{\delta||\phi(x_i)||}(0,z)$ with smooth rotation-invariant kernels, bounded support and bandwidth $h$:
\begin{equation}\label{eq:kernel} k_h(z_i,z)= \frac{1}{h^m}\cdot k \left (\frac{||z_i-z||}{h} \right)  \cdot \mathds{1}_{||z_i-z||<h}  \end{equation}
with $\mathds{1}_{||z_i-z||<h}=1$ when $||z-z_i||< h$ and $0$ otherwise, and such that $\int_{z\in \RR^m} k_h(z_0,z)\ dz=1\text{ for all $z_0$} .$ 
An example for such a kernel is a truncated isotropic normal distribution with variance $h^2 \sigma^2I$, 
$k_h(z_i,z)= \mathcal{N}(z_i, h^2\sigma^2)(z)$.
%\begin{equation}\label{eq:scaledTruncatedGaussians}
%     k_h(z_i,z)= \mathcal{N}(z_i, h^2\sigma^2)(z) =  \frac{C}{h^m} \cdot \exp\left(- \frac{||z-z_i||^2}{h^2\sigma^2}\right )  \cdot \mathbbm{1}_{||z_i-z||<h} 
% \end{equation}
%where $C$ is a normalizing constant, or the truncated multivarite Epanechnikov kernel (see e.g.~\cite{silverman1986density}). 
The following result states that the densities in \eqref{eq:kernel} can indeed be induced by distributions on feature matrices, which will enable us to connect feature robustness with $\epsilon$-representativeness. 

\begin{proposition}\label{prp:lambdaA}
Let $S^\phi=\{\phi(x_i)\ |x_i\in S\}$ be a set of feature vectors in $\RR^m$. With $k_h$ defined as in \eqref{eq:kernel}, let $\lambda_i(z)=k_{\delta||\phi(x_i)||}(0,z)$ define an $|S|$-tuple $\Lambda_{\delta}$ of densities. Then there exists a distribution $\mathcal{A}_\delta$ on matrices in $\RR^{m\times m}$ of norm less than $\delta$
%that induces  each $\lambda_i$ by multiplication with $\phi(x_i)$ so that $\Lambda_{\delta,\mathcal{A}}=\Lambda$.
such that for each $i=1,\ldots,|S|$,
\[ \Exp{A\sim \mathcal{A}_\delta}{  \loss(\psi(\phi_A(x_i)),y(\phi_A(x_i)))}= \mathbb{E}_{\xi \sim \lambda_i}\Big[\loss( \psi(\phi(x_i)+ \xi),y(\phi(x_i)+\xi)) \Big] \]
%\[ \Exp{A\sim \mathcal{A}_\delta}{  \loss(\psi(\phi(x_i)+A\phi(x_i)),y[\phi(x_i)+A\phi(x_i)])}= \mathbb{E}_{\xi \sim \lambda_i}\Big[\loss( \psi(\phi(x_i)+ \xi),y[\phi(x_i)+\xi]) \Big] \]
\end{proposition}

The technical proof is deferred to the appendix, but we describe the distribution $\mathcal{A}_\delta$ on matrices for later use: The desired distribution is defined on the set of matrices of the form $rO$ for a real number $r$ and an orthogonal matrix $O$ (i.e. $OO^T=O^TO=I$) as a product measure combining the (unique) Haar measure on the set of orthogonal matrices $\mathcal{O}(m)$ with a suitable distribution on $\RR$. The Haar measure on $\mathcal{O}(m)$ induces the uniform measure on a sphere of radius $r$ via multiplication with a vector of length $r$ \citep{Krantz}, and we choose a  measure on $\RR$ to match the radial change of the kernel $k_h$.

\section{Relative Flatness of the Loss Surface}\label{sct:FRandFlatness}
Flatness is a property of the parameter space quantifying the change in loss under small parameter perturbations, classically measured by the trace of the loss Hessian $Tr(H)$, where $H$ is the matrix containing the partial second derivatives of the empirical risk with respect to  all parameters of the model. 
%
%We decomposed the generalization gap into feature robustness and $\epsilon$-representativeness in the corresponding feature space \eqref{eq:splitError}. 
In order to connect feature robustness (a property of the feature space) to flatness, 
we present how key equation \eqref{eq:keyEquation} translates to the empirical risk:
%we consider key equation~\eqref{eq:keyEquation}.
%In order to connect these two properties, we consider key equation \eqref{eq:keyEquation}.
%We present how \eqref{eq:keyEquation} translates to the empirical risk: 
For a model $f(x,\w)=\psi(\w,\phi(x))=g(\w\phi(x))$ with parameters $\w\in \RR^{d\times m}$ and $g:\RR^d\rightarrow \Ycal$ a function on a matrix product of parameters $\w$ and a feature representation $\phi:\Xcal\rightarrow \RR^m$ and any feature matrix $A\in\RR^{m\times m}$ we have that 
\begin{equation}\label{eq:InputToParameters}
\begin{split}
    \Ecal_{emp}(\w+ &\w A,\phi(S)) =\frac 1{|S|} \sum_{i=1}^{|S|} \loss( \psi (\w+ \w A,\phi(x_i)),y_i) \\
    &= \frac 1{|S|} \sum_{i=1}^{|S|} \loss(\psi(\w, \phi(x_i)+A\phi(x_i)),y_i) = \frac 1{|S|} \sum_{i=1}^{|S|} \loss(\psi(\w, \phi_A(x_i)),y_i)
\end{split}
\end{equation}
Subtracting $\Ecal_{emp}(\w,\phi(S))=\frac 1{|S|} \sum_{i=1}^{|S|} \loss(\psi(\w, \phi(x_i)),y_i)$, we can recognize feature robustness \eqref{eq:defRobustness} on the right side of this equality when labels are constant under perturbations of the features, i.e.\ $y(\phi_A(x_i))=y_i$. In other words, flatness $\Ecal_{emp}(\w+ \mathbf{v},\phi(S))-\Ecal_{emp}(\w,\phi(S))$ 
%for some $\mathbf{v}=\w A$, 
describes the performance of a model function on perturbed feature vectors while holding labels constant. 
%This connection between feature robustness and empirical loss surface assumes that the performance for perturbed features is measured with locally constant labels. It can be seen from \eqref{eq:InputToParameters} that, in general, flatness measures of the loss surface that consider the change of the empirical loss as a function of the weights similarly assume locally constant labels. 
We proceed to introduce a novel, but natural, loss Hessian-based flatness measure that approximates feature robustness, given that the underlying data distribution $\Dcal$ satisfies the assumption of locally constant labels.

%This is the key equation allowing to connect properties in the sample space around training points to the loss as a function of the weight parameters, which can be bounded by flatness measures of the loss surface. For this, let $\loss$ be a loss function and $f(x,\w)=\psi(\phi(x))=g(\w\phi(x))$ be a model with $\w\in \RR^{d\times m}$ and $g:\RR^d\rightarrow \Ycal$ an arbitrary twice differentiable function on a matrix product of parameters $\w$ and the image of $x$ under a feature representation $\phi:\Xcal\rightarrow \RR^m$. Let $\w_i\in\RR^{1\times m}$ denote a row of weight matrix $\w$. 
With $\w_s=(w_{s,t})_t \in \RR^{1\times m}$ denoting the $s$-th row of the parameter matrix $\w$, we let $H_{s,s'}(\w,\phi(S))\in\RR^{m\times m}$ denote the Hessian matrix containing all partial second derivatives of the empirical risk $\Ecal_{emp}(\w,\phi(S))$ with respect to weights in rows $\w_s$ and $\w_{s'}$,  i.e. 
\begin{equation}\label{eq:hessians}
    H_{s,s'}(\w,\phi(S))=\left [\frac{\partial^2 \mathcal{E}_{emp}(\w,\phi(S))}{\partial w_{s,t} \partial w_{s',t'}} \right ]_{1\leq t,t'\leq m}.
\end{equation}

\begin{definition}\label{def:flatnessMeasure}
For a model $f(x,\w)=g(\w\phi(x))$, $\w\in\RR^{d\times m}$, with a twice differentiable function $g$, a twice differentiable loss function $\loss$ and a sample set $S$, relative flatness is defined by
 \begin{equation}\label{eq:defFlatness}
% \kappa^\phi(\w) := \sum_{s=1}^d || \w_s||^2  \cdot \lambda_{max}(H_{s,s}(\w,\phi(S)))\text{ and }
 \kappa^\phi_{Tr}(\w) := \sum_{s,s'=1}^d \langle \w_s,\w_{s'}\rangle  \cdot Tr(H_{s,s'}(\w,\phi(S))),
\end{equation}
where  $Tr$ denote the  trace and  $\langle \w_s,\w_{s'}\rangle=\w_s\w_{s'}^T$ the scalar product of two row vectors.

%$\lambda_{max}$ and $Tr$ denote the maximal eigenvalue and unnormalized trace respectively and  $\langle \w_s,\w_{s'}\rangle=\w_s\w_{s'}^T$ the scalar product between two row vectors.
\end{definition}

%\begin{definition}\label{def:flatnessMeasure}
%Let $\loss$ be a loss function and $f(x,\w)=\psi(\phi(x))=g(\w\phi(x))$ be a model with $\w\in \RR^{d\times m}$ and $g:\RR^d\rightarrow \Ycal$ an arbitrary twice differentiable function on a matrix product of parameters $\w$ and the image of $x$ under a feature representation $\phi:\Xcal\rightarrow \RR^m$. With $H\mathcal{E}_{emp}(\vvec \w,S)$ denoting the Hessian of the empirical risk as a function of $\w$, we define two relative flatness measures of the loss surface based on the maximal eigenvalue $\lambda_{max}$ and the unnormalized trace $Tr$ as
% \begin{equation}\label{eq:defFlatness}
% \kappa^\phi(\w) := ||\vvec \w||_F^2\cdot \lambda_{max}(H\mathcal{E}_{emp}(\vvec \w,S))\text{ and } \kappa^\phi_{Tr}(\w) := ||\vvec \w||_F^2\cdot Tr(H\mathcal{E}_{emp}(\vvec \w,S)).
%\end{equation}
%\end{definition}

\paragraph{Properties of relative flatness} 
(i) Relative flatness simplifies to ridge regression for linear models $f(x,\w)=\w x\in\RR$ ($\Xcal=\RR^d$, $g=id$ and $\phi=id$) and squared loss: To see this, note that for any loss function $\loss$, the second derivatives with respect to the parameters $\w\in \RR^d$ computes to 
$\frac{\partial^2 \loss}{\partial w_i \partial w_j} = \frac{\partial^2 \loss}{\partial (f(x,\w))^2} x_i x_j.$
For $\loss(\hat y,y)=||\hat y-y||^2$ the squared loss function, $\sfrac{\partial^2 \loss}{\partial \hat y^2}=2$ and the Hessian is independent of the parameters $\w$. In this case, $\kappa_{Tr}^{id}=c\cdot ||\vvec \w||^2$ with a constant $c=\sum_{x\in S}2Tr(xx^T)$, which is the well-known Tikhonov (ridge) regression penalty.

(ii) Invariance under reparameterization: We consider neural network functions 
\begin{equation}
f(x)=\w^L \sigma (\ldots \sigma(\w^2 \sigma(\w^1 x+b^1)+b^2)\ldots )+b^L
%f(x)=\w^L \sigma (\ldots \sigma(\w^l\sigma(\w^{l-1}\sigma(\ldots  \sigma(\w^1 x+b^1))\ldots)+b^{l-1})+b^l)\ldots )+b^L
\end{equation} of a neural network of $L$ layers with nonlinear activation function $\sigma$. 
%We hide a possible non-linearity at the output by integrating it in a loss function $\loss$ chosen for neural network training. 
By letting $\phi^l(x)$
%$\phi^l(x)=\sigma ( \w^{l-1} \sigma (\ldots \sigma(\w^2 \sigma(\w^1 x+b^1)+b^2)\ldots)+b^{l-1})$ 
denote the composition of the first $l-1$ layers, 
we obtain a decomposition $f(x,\w^l)=g^l(\w^l \phi^l(x))$ of the network. % function
%and $g^l(z)$ the composition of the activation function of the $l$-th layer together with the rest of layers
%=\w^L \sigma (\ldots \sigma(z+b^{l})\ldots )+b^L$
%the composition of the activation function of the $l$-th layer together with the rest of layers
%, we can write for each layer $l$, $f(x,\w^l)=g^l(\w^l \phi^l(x))$. 
Using \eqref{eq:defFlatness} we obtain a relative flatness measure $\kappa^l_{Tr}(\w)$ % := \sum_{s,t=1}^d \langle \w^l_s,\w^l_t \rangle \cdot Tr(H_{s,t}(\w,S))$ 
for the chosen layer.% a measure of relative flatness: Let .
%\begin{equation}\label{eq:flatnessNetworks}
%    \kappa^l(\w) := \sum_{s=1}^d ||\w^l_{s}||^2 \cdot \lambda_{max}(H_{s,s}(\w^l_s,S)) \text{ and }
%    \kappa^l_{Tr}(\w) := \sum_{s,t=1}^d \langle \w^l_s,\w^l_t \rangle \cdot Tr(H_{s,t}(\w,S))
%\end{equation}
%with $Tr$ calculating the trace of the Hessian $H_{s,s}$ of the empirical risk with respect to rows $\w_s^l$ of the parameters of the $l$-th.% layer denoted by . 

For a well-defined Hessian of the loss function, we require the network function to be twice differentiable. With the usual adjustments (equations only hold almost everywhere in parameter space), we can also consider neural networks with ReLU activation functions. 
In this case, \citet{dinhSharp} noted that the network function ---and with it the generalization performance--- remains unchanged under linear reparameterization, i.e., multiplying layer $l$ with $\alpha>0$ and dividing layer $k\neq l$ by $\alpha$, but common measures of the loss Hessian change. 
%
% \w^l\rightarrow \lambda \w^l$ for $\lambda>0$,
%
% can lead to the same network function by simultaneously multiplying another layer by the inverse of $\lambda$, $\w^k\rightarrow \sfrac 1\lambda \w^k,\ k\neq l$.  Representing the same function, the generalization performance remains unchanged, but this linear reparameterization changes the common measures of the Hessian of the loss, which constitutes an issue in relating flatness to generalization. 
Our measure fixes this issue in relating flatness to generalization since the change of the loss Hessian is compensated by multiplication with the scalar products of weight matrices and is therefore invariant under layer-wise reparameterizations~\citep[cf.][]{exploringGen}. It is also invariant to neuron-wise reparameterizations, i.e., multiplying all incoming weights into a neuron by a positive number $\alpha$ and dividing all outgoing weights by $\alpha$~\citep{pathSGD}, except for neuron-wise reparameterizations of the feature layer $\phi^l$. Using a simple preprocessing step (a neuron-wise reparameterization with the variance over the sample)%choosing a natural representative by normalizing the weights in the feature layer with the variance over the sample)
, our proposed measure becomes independent of all neuron-wise reparameterizations.

\begin{theorem}\label{thm:reparameterizations}
Let $\sigma_i$ denote the variance of the i-th coordinate of $\phi^l(x)$ over 
samples $x\in S$ and $V=diag\left ( {\sigma_1},\ldots, {\sigma_{n_{l-1}}}\right )$. If the relative flatness measure $\kappa^l_{Tr}$ is applied to the representation %\[f=\psi(\w^lV,b^l,\w^{l+1},b^{l+1}\ldots, \w^L,b^L) \circ  \phi(V^{-1}\w^{l-1},V^{-1}b^{l-1},\ldots,\w^1,b^1),\]
%\[f(x)=f(x;\w^1,b^1,\ldots,V^{-1}\w^{l-1},V^{-1}b^{l-1},\w^lV,b^l,\w^{l+1},b^{l+1}\ldots, \w^L,b^L) \]
\begin{equation*}f(x)=\w^L \sigma (\ldots \sigma(\w^lV\ \sigma(V^{-1}\w^{l-1}\sigma(\ldots  \sigma(\w^1 x+b^1))\ldots)+V^{-1}b^{l-1})+b^l)\ldots )+b^L\end{equation*}
then $\kappa^l_{Tr}$ is invariant under all neuron-wise (and layer-wise) reparameterizations 
\end{theorem}

We now connect flatness with feature robustness: Relative flatness approximates feature robustness for a model at a local minimum of the empirical risk, when labels are approximately constant in neighborhoods of the training samples $(\phi(x),y)\in \phi(S)$ in feature space. 
%This differs from layer-wise reparameterization invariant flatness $\|w\|Tr(H)$~\citep{exploringGen}, and is more natural than the optimization-based normalized sharpness~\citep{tsuzuku}. 

\begin{theorem}\label{thm:averaging}
Consider a model $f(x,\w)=g(\w\phi(x))$ as above, a loss function $\loss$ and a sample set $S$, and let $O_m\subset \RR^{m\times m}$ denote the set of orthogonal matrices. Let $\delta$ be a positive (small) real number and $\w=\omega\in\RR^{d\times m}$ denote parameters at a local minimum of the empirical risk on a sample set $S$. If the labels satisfy that 
$y(\phi_{\delta A}(x_i))=y(\phi(x_i))=y_i$
%$y[\phi(x_i+\delta A\phi(x_i))]=y(\phi(x_i))=y_i$
for all $(x_i,y_i)\in S$ and all $||A||\leq 1$, then 
%(i) for each feature selection matrix $||A||\leq 1$ the model $f(x,\omega)$ is $\left ( (\delta, S, A), \epsilon \right )$-feature robust for $\epsilon= \frac{\delta^2 d}2 \kappa^\phi(\omega) +\mathcal{O}(\delta^3)$, and  
%(ii)
 $f(x,\omega)$ is $\left ( (\delta, S, O_m), \epsilon \right )$-feature robust on average over $O_m$  for $\epsilon= \frac{\delta^2}{2m} \kappa^\phi_{Tr}(\omega)+\mathcal{O}(\delta^3)$. 
\end{theorem}
Applying the theorem to Eq.~\ref{eq:splitError} implies that if the data is representative, i.e., $\Ecal_{Rep}^\phi(f,S,\Lambda_{\mathcal{A}_\delta})\approx 0$ for the distribution $\mathcal{A}_\delta$ of Prop.~\ref{prp:lambdaA}, then $\Ecal_{gen}(f(\cdot, \omega),S)\lesssim \frac{\delta^2}{2m} \kappa^\phi_{Tr}(\omega)+\mathcal{O}(\delta^3)$.
The assumption on locally constant labels in Thm.~\ref{thm:averaging} can be relaxed to approximately locally constant labels without unraveling the theoretical connection between flatness and feature robustness. Appendix~\ref{appdx:locallabelChanges} investigates consequences from even dropping the assumption of approximately locally constant labels.

\section{Flatness and Generalization}\label{sec:FRandGen}
Combining the results from sections~\ref{sct:representativeness}--\ref{sct:FRandFlatness}, we connect flatness to the generalization gap when the distribution can be represented by smooth probability densities on a feature space with approximately locally constant labels. By approximately locally constant labels we mean that, for small $\delta$, the loss in $\delta||\phi(x_i)||$-neighborhoods around the feature vector of a training sample $x_i$ is approximated (on average over all training samples) by the loss for constant label $y(x_i)$ on these neighborhoods. %when $y(x)=y(x_i)$ for $x$ from these neighborhoods. 
This and the following theorem connecting flatness and generalization are made precise in Appendix~\ref{proof:thm6}.%\ref{appdx:proof:thm:FlatnessGenBound}.

\begin{theorem}[\textit{informal}]\label{thm:FlatnessGenBound} Consider a model $f(x,\w)=g(\w\phi(x))$ as above, a loss function $\loss$ and a sample set $S$,  let $m$ denote the dimension of the feature space defined by $\phi$ and let $\delta$ be a positive (small) real number. Let $\omega$ denote a local minimizer of the empirical risk on a sample set $S$. If the distribution $\Dcal$ has a smooth density $p^\phi_\Dcal$ on the feature space $\RR^m$ with approximately locally constant labels around the points $x\in S$, then it holds with probability $1-\Delta$ over sample sets $S$ that
\[\Ecal_{gen}(f(\cdot, \omega),S)\lesssim |S|^{-\frac{2}{4+m}} \left ( \frac{\kappa^\phi_{Tr}(\omega)}{2m}  +  C_1(p^\phi_\Dcal,L)+\frac{C_2(p^\phi_\Dcal,L)}{\sqrt{\Delta}} \right ) \] 
up to higher orders in $|S|^{-1}$ for constants $C_1,C_2$ that depend only on the distribution in feature space $p^\phi_\Dcal$ induced by $\phi$, the chosen $|S|$-tuple $\Lambda_{\delta}$ and the maximal loss $L$.
\end{theorem}

%The proof of the theorem builds upon \eqref{eq:splitError} stating that the generalization gap $\Ecal_{gap}$ is approximated by feature robustness if the data's feature representation satisfies that $\Ecal_{Rep}^\phi(f,S,\Lambda_{\Acal_\delta})\approx 0$ for some distribution $\Acal_\delta$ on feature matrices. %It remains to bound representativeness 
%Decreasing the $\delta$-neighborhoods around training data for growing sample size, we use kernel density estimation (KDE) to bound representativeness for the distribution $\Acal_\delta$ on feature matrices from Proposition~\ref{prp:lambdaA}. The same distribution $\Acal_\delta$ is suitable to apply Theorem~\ref{thm:averaging} to bound feature robustness by our proposed measure of relative flatness from Definition~\ref{def:flatnessMeasure}, leading to the bound on the generalization gap. %Figure~\ref{} illustrates how the theoretical components connect. 

To prove Theorem~\ref{thm:FlatnessGenBound} we bound both $\epsilon$-representativeness and feature robustness in Eq.~\ref{eq:splitError}. For that, the main idea is that the family of distributions considered in Proposition~\ref{prp:lambdaA} has three key properties: (i) it provides an explicit link between the distributions on feature matrices $\mathcal{A}_\delta$ 
used in feature robustness and the family of distributions $\Lambda_\delta$ 
of $\epsilon$-representativeness (Proposition~\ref{prp:lambdaA}) (ii) it allows us to bound feature robustness using Thm.~\ref{thm:averaging}; and (iii) it is simple enough that it allows us to use standard results of kernel density estimation (KDE) to bound representativeness.

Our bound suffers from the curse of dimensionality, but for the chosen feature space instead of the (usually much larger) input space. The dependence on the dimension is a result of using KDE uniformly over all distributions satisfying mild regularity assumptions. In practice, for a given distribution and sample set $S$, representativeness can be much smaller, which we showcase in a toy example in Sec.~\ref{sec:experiments}. In the so-called interpolation regime, where datasets with arbitrarily randomized labels can be fit by the model class, the obtained convergence rate is consistent with the no free lunch theorem and the convergence rate derived by \citet{belkinPerfectFit} for an interpolation technique using nearest neighbors.
%with $\alpha=2$
 
A combination of our approach with prior assumptions on the hypotheses or the algorithm in accordance to statistical learning theory could potentially achieve faster convergence rate. Our herein presented theory is instead based solely on interpolation and aims to understand the role of flatness (a local property) in generalization: If the data is representative in feature layers and if the distribution can be approximated by locally constant labels in these layers, then flatness of the empirical risk surface approximates the generalization gap. Conversely, Equation~\ref{eq:InputToParameters} shows that flatness measures the performance under perturbed features only when labels are kept constant. As a result, we offer an explanation for the often observed correlation between flatness and generalization: Real-world data distributions for classification are benign in the sense that small perturbations in feature layers do not change the target class, i.e., they can be approximated by locally constant labels. (Note that the definition of adversarial examples hinges on this assumption of locally constant labels.) In that case, feature robustness is approximated by flatness of the loss surface. If the given data and its feature representation are further $\epsilon$-representative for small $\epsilon\approx 0$, then flatness becomes the main contributor to the generalization gap leading to their noisy, but steady, correlation. 
\section{Empirical Validation}\label{sec:experiments}
\begin{figure}[t]
    \centering
    \begin{minipage}{.49\textwidth}
        \centering
        \includegraphics[height=4.3cm]{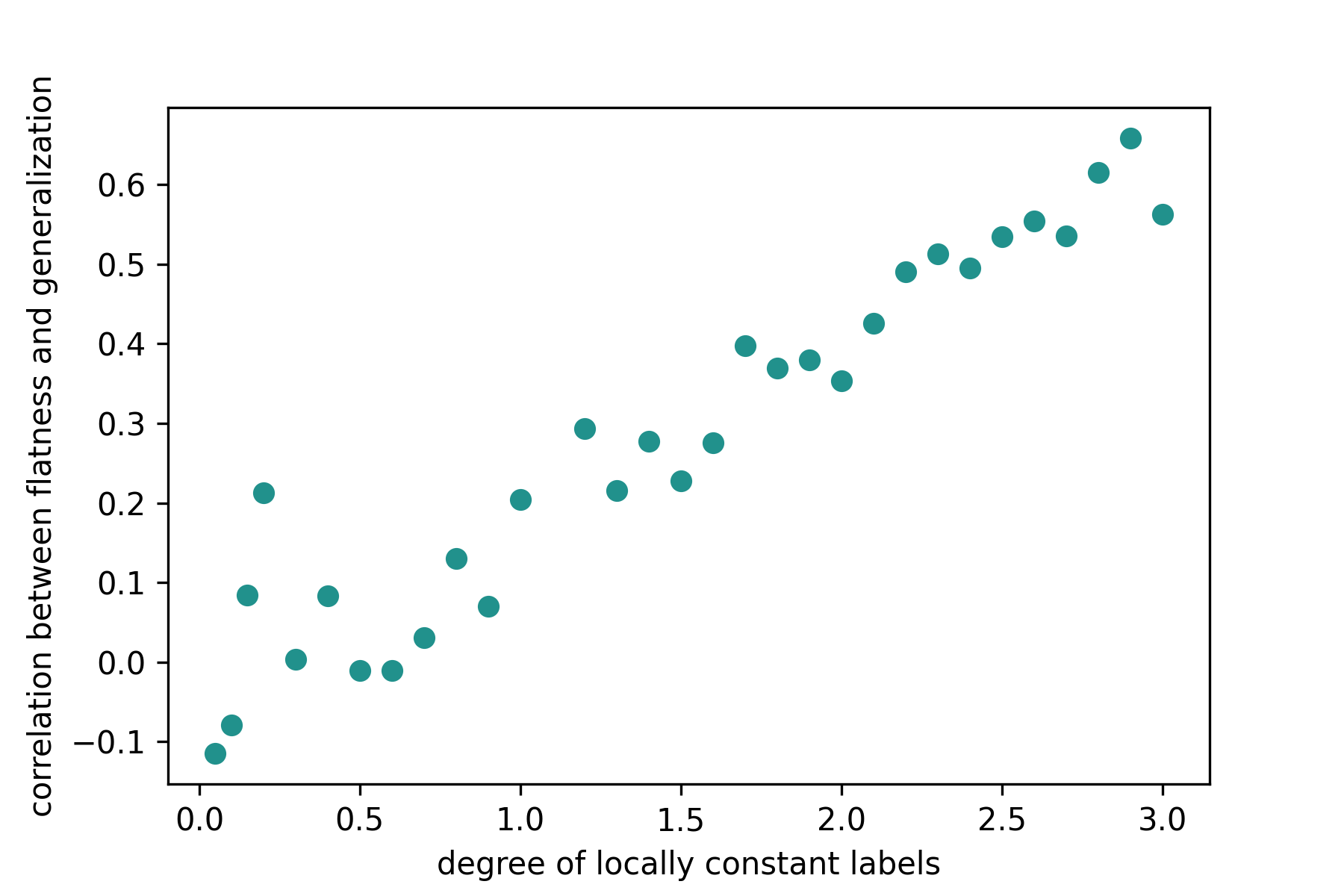}
        \caption{The correlation between flatness and generalization increases with the degree of locally constant labels.}
        \label{fig:locallyConstantLabels}
    \end{minipage}
    \hspace{0.1cm}
    \begin{minipage}{0.49\textwidth}
        \centering
        \includegraphics[height=4.3cm]{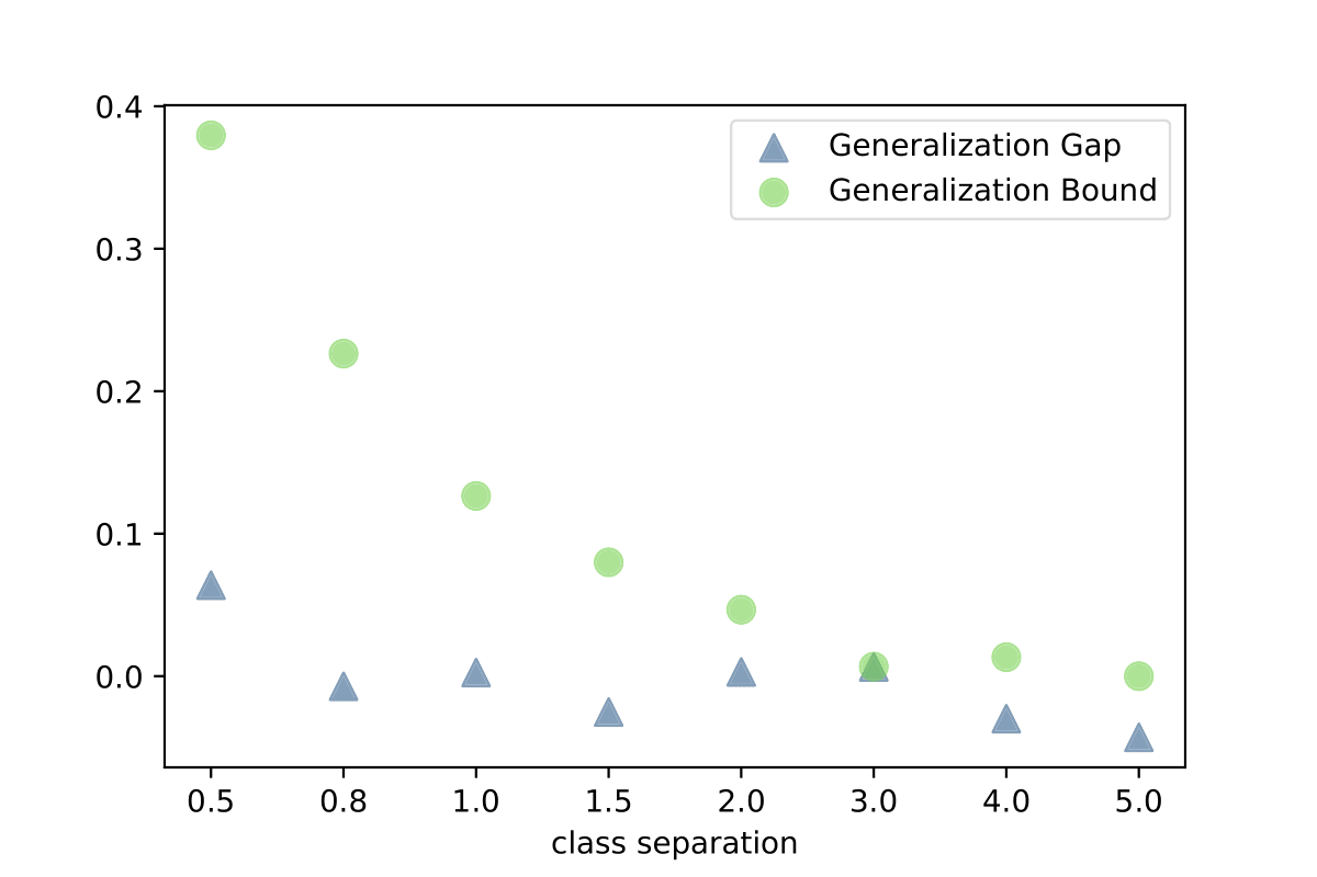}
    \caption{Approximation of representativeness via KDE together with relative flatness leads to a tight generalization bound.}
    \label{fig:approxRepresentativeness}
    \end{minipage}
\end{figure}
We empirically validate the assumptions and consequences of the theoretical results derived above~\footnote{~Code is available at~\url{\experimentURL}.}. For that, we first show on a synthetic example that the empirical correlation between flatness and generalization decreases if labels are not locally constant, up to a point when they are not correlated anymore. We then show that the novel relative flatness measure correlates strongly with generalization, also in the presence of reparameterizations. Finally, we show in a synthetic experiment that while representativeness cannot be computed without knowing the true data distribution, it can in practice be approximated. This approximation---although technically not a bound anymore---tightly bounds the generalization gap.
Synthetic data distributions for binary classification are generated by sampling $4$ Gaussian distributions in feature space (two for each class) with a given distance between their means (class separation). We then sample a dataset in feature space $S^\phi$, train a linear classifier $\psi$ on the sample, randomly draw the weights of a 4-layer MLP $\phi$, and generate the input data as $S=(\phi^{-1}(S^\phi_x),S^\phi_y)$. This yields a dataset $S$ and a model $f=\phi\circ\psi$ such that $\phi(S)$ has a given class separation. 
Details on the experiments are provided in Appdx.~\ref{appdx:experiments_details}.
% \begin{figure}
%     \centering
%     \includegraphics[width=8cm]{images/experiments/locallyConstantLabels/correlation.png}
%     \caption{Caption}
%     \label{fig:locallyConstantLabels}
% \end{figure}
% \begin{figure}
%     \centering
%     \includegraphics[width=8cm]{images/experiments/flatness/cifar_reparam_all_measures_trace.png}
%     \caption{Caption}
%     \label{fig:relativeFlatness}
% \end{figure}
% \begin{figure}
%     \centering
%     \includegraphics[width=8cm]{images/experiments/approximatingRepresentativeness/genBoundApprox.pdf}
%     \caption{Caption}
%     \label{fig:approxRepresentativeness}
% \end{figure}

\textbf{Locally constant labels:} To validate the necessity of locally constant labels, we measure the correlation between the proposed relative flatness measure and the generalization gap for varying degrees of locally constant labels, as measured by the class separation on the synthetic datasets.
For each chosen class separation, we sample $100$ random datasets of size $500$ on which we measure relative flatness and the generalization gap. Fig.~\ref{fig:locallyConstantLabels} shows the average correlation for different degrees of locally constant labels, showing that the higher the degree, the more correlated flatness is with generalization. If labels are not locally constant, flatness does not correlate with generalization.

\begin{wrapfigure}{r}{9cm}
        \centering
        \includegraphics[width=9cm]{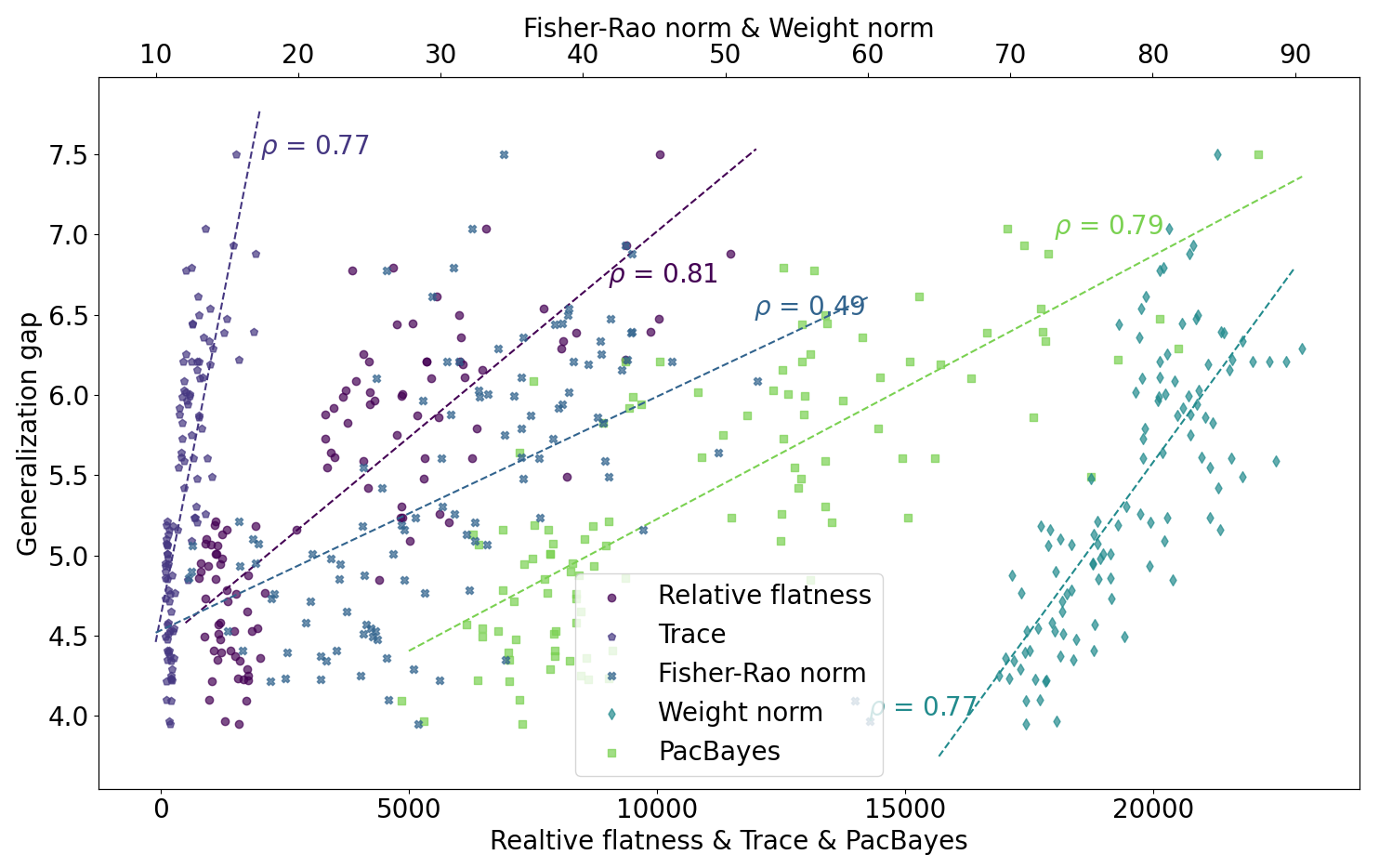}
        \caption{The generalization gap for various local minima correlates stronger with relative flatness than standard flatness, Fisher-Rao norm, PacBayes based measure and weights norm (points corresp. to local minima).}
        \label{fig:relativeFlatness}
\end{wrapfigure}
\textbf{Approximating representativeness:} While representativeness cannot be calculated without knowing the data distribution, it can be approximated from the training sample $S$ by the error of a density estimation on that sample. For that, we use multiple random splits of $S$ into a training set $S_{\text{train}}$ and a test set $S_{\text{test}}$, train a kernel density estimation on $S_{\text{train}}$ and measure its error on $S_{\text{test}}$. Again, details can be found in Appx.~\ref{appdx:experiments_details}. The lower the class separation of the synthetic datasets, the harder the learning problem and the less representative a random sample will be. For each sample and its distribution, we compute the generalization gap and the approximation to the generalization bound.
The results in Fig.~\ref{fig:approxRepresentativeness} show that the approximated generalization bound tightly bounds the generalization error (note that this approximation is technically not a bound anymore). Moreover, as expected, the bound decreases the easier the learning problems become.

\textbf{Relative flatness correlates with generalization:} We validate the correlation of relative flatness to the generalization gap in practice by measuring it for $110$ different local minima---achieved via different learning setups, such as initialization, learning rate, batch size, and optimization algorithm%, as well as reparameterizations
---of LeNet5~\citep{lecun-mnisthandwrittendigit-2010} on CIFAR10~\citep{Krizhevsky09learningmultiple}. We compare this correlation to the classical Hessian-based flatness measures using the trace of the loss-Hessian, the Fisher-Rao norm~\cite{fisherRao}, the PACBayes flatness measure that performed best in the extensive study of~\citet{jiang2020fantastic} and the $L_2$-norm of the weights.
%(the sharpness measure of \citet{keskarLarge} is not applicable to ReLU networks~\citep{dinhSharp}). 
The results in Fig.~\ref{fig:relativeFlatness} show that indeed relative flatness has %substantially 
higher correlation than all the competing measures.
%and confirms that the $L_2$-norm alone is not a suitable indicator for generalization. 
Of these measures, only the Fisher-Rao norm is reparameterization invariant but shows the weakest correlation in the experiment. In Appdx~\ref{appdx:experiments_details} we show how reparameterizations of the network significantly reduce the correlation for non-reparameterization invariant measures.

\section{Discussion and Conclusion}
Contributing to the trustworthiness of machine learning, this paper provides a rigorous connection between flatness and generalization.
 As to be expected for a local property, our association between flatness and generalization requires the samples and its representation in feature layers to be representative for the target distribution. But our derivation uncovers a second, usually overlooked condition. Flatness of the loss surface measures the performance of a model close to training points when labels are kept locally constant. If a data distribution violates this, then flatness cannot be a good indicator for generalization. %The more benign a data distribution is, the more representative training samples are and thus generalization is more and more determined by relative flatness.

Whenever we consider feature representations other than the input features, the derivation of our results makes one strong assumption: the existence of a target output function $y(\phi(x))$ on the feature space $\phi(\Xcal)$. By moving assumptions on the distribution from the input space to the feature space, we achieve a bound based on interpolation that depends on the dimension of the feature layer instead of the input space. Hence, we assume that the feature representation is reasonable and does not lose information that is necessary for predicting the output. To achieve faster convergence rates independent of any involved dimensions, future work could aim to combine our approach of interpolation with a prior-based approach of statistical learning theory.

Our measure of relative flatness may still be improved in future work. Better estimates for the generalization gap are possible by improving the representativeness of local distributions in two ways: The support shape of the local distributions can be improved and their volume-parameter $\delta$ can be optimally chosen.
Both improvements will affect the derivation of the measure of relative flatness as an estimation of feature robustness for the corresponding distributions on feature matrices. Whereas different support shapes change the trace to a weighted average of the Hessian eigenvalues,
the volume parameter can provide a correcting scaling factor. 
Both approaches seem promising to us, as our relative measure from Definition~\ref{def:flatnessMeasure} already outperforms the competing measures of flatness in our empirical validation.

\section*{Acknowledgements}
Cristian Sminchisescu was supported by the European Research Council Consolidator grant SEED, CNCS-UEFISCDI (PN-III-P4-ID-PCE-2016-0535, PN-III-P4-ID-PCCF-2016-0180), the EU Horizon 2020 grant DE-ENIGMA (688835), and SSF.

Mario Boley was supported by the Australian Research Council (under DP210100045).

We would like to thank Julia Rosenzweig, Dorina Weichert, Jilles Vreeken, Thomas G\"artner, Asja Fischer, Tatjana Turova and Alexandru Aleman for the great discussions.% on probability theory and estimation.

%\vfill
%\pagebreak
%\newpage
%\input{impactStatement}

%% If we experience again bibliography issues where all references are broken
%% just uncomment the following line and recompile
%\typeout{}

\bibliographystyle{plainnat}
\bibliography{bibliography}

\begin{thebibliography}{44}
\providecommand{\natexlab}[1]{#1}
\providecommand{\url}[1]{\texttt{#1}}
\expandafter\ifx\csname urlstyle\endcsname\relax
  \providecommand{\doi}[1]{doi: #1}\else
  \providecommand{\doi}{doi: \begingroup \urlstyle{rm}\Url}\fi

\bibitem[Bartlett et~al.(2020)Bartlett, Long, Lugosi, and
  Tsigler]{bartlett2020benign}
Peter~L Bartlett, Philip~M Long, G{\'a}bor Lugosi, and Alexander Tsigler.
\newblock Benign overfitting in linear regression.
\newblock \emph{Proceedings of the National Academy of Sciences}, 117\penalty0
  (48):\penalty0 30063--30070, 2020.

\bibitem[Belkin et~al.(2018)Belkin, Hsu, and Mitra]{belkinPerfectFit}
Mikhail Belkin, Daniel~J Hsu, and Partha Mitra.
\newblock Overfitting or perfect fitting? risk bounds for classification and
  regression rules that interpolate.
\newblock In \emph{Advances in Neural Information Processing Systems}, pages
  2300--2311, 2018.

\bibitem[Belkin et~al.(2019)Belkin, Hsu, Ma, and Mandal]{belkin2019reconciling}
Mikhail Belkin, Daniel Hsu, Siyuan Ma, and Soumik Mandal.
\newblock Reconciling modern machine-learning practice and the classical
  bias--variance trade-off.
\newblock \emph{Proceedings of the National Academy of Sciences}, 116\penalty0
  (32):\penalty0 15849--15854, 2019.

\bibitem[Chaudhari et~al.(2017)Chaudhari, Choromanska, Soatto, LeCun, Baldassi,
  Borgs, Chayes, Sagun, and Zecchina]{entropySGD}
Pratik Chaudhari, Anna Choromanska, Stefano Soatto, Yann LeCun, Carlo Baldassi,
  Christian Borgs, Jennifer~T. Chayes, Levent Sagun, and Riccardo Zecchina.
\newblock Entropy-sgd: Biasing gradient descent into wide valleys.
\newblock In \emph{Proceedings of the International Conference of Learning
  Representations}, 2017.

\bibitem[Dinh et~al.(2017)Dinh, Pascanu, Bengio, and Bengio]{dinhSharp}
Laurent Dinh, Razvan Pascanu, Samy Bengio, and Yoshua Bengio.
\newblock Sharp minima can generalize for deep nets.
\newblock In \emph{Proceedings of the 34th International Conference on Machine
  Learning}, volume~70, pages 1019--1028. JMLR. org, 2017.

\bibitem[Dziugaite and Roy(2017)]{dzigugaite}
Gintare~Karolina Dziugaite and Daniel~M Roy.
\newblock Computing nonvacuous generalization bounds for deep (stochastic)
  neural networks with many more parameters than training data.
\newblock \emph{AAAI}, 2017.

\bibitem[Dziugaite and Roy(2018)]{dziugaite2018data}
Gintare~Karolina Dziugaite and Daniel~M Roy.
\newblock Data-dependent pac-bayes priors via differential privacy.
\newblock In \emph{Advances in Neural Information Processing Systems}, pages
  8430--8441, 2018.

\bibitem[Foret et~al.(2021)Foret, Kleiner, Mobahi, and
  Neyshabur]{foret2021sharpnessaware}
Pierre Foret, Ariel Kleiner, Hossein Mobahi, and Behnam Neyshabur.
\newblock Sharpness-aware minimization for efficiently improving
  generalization.
\newblock In \emph{Proceedings of the International Conference on Learning
  Representations}, 2021.

\bibitem[Glorot and Bengio(2010)]{glorot2010understanding}
Xavier Glorot and Yoshua Bengio.
\newblock Understanding the difficulty of training deep feedforward neural
  networks.
\newblock In \emph{Proceedings of the Thirteenth International Conference on
  Artificial Intelligence and Statistics}, pages 249--256. PMLR, 2010.

\bibitem[Hochreiter and Schmidhuber(1995)]{simplifyingByFlat}
Sepp Hochreiter and J{\"u}rgen Schmidhuber.
\newblock Simplifying neural nets by discovering flat minima.
\newblock In \emph{Advances in Neural Information Processing Systems}, pages
  529--536, 1995.

\bibitem[Hochreiter and Schmidhuber(1997)]{flatMinima}
Sepp Hochreiter and J{\"u}rgen Schmidhuber.
\newblock Flat minima.
\newblock \emph{Neural Computation}, 9\penalty0 (1):\penalty0 1--42, 1997.

\bibitem[Izmailov et~al.(2018)Izmailov, Podoprikhin, Garipov, Vetrov, and
  Wilson]{izmailov2018averaging}
Pavel Izmailov, Dmitrii Podoprikhin, Timur Garipov, Dmitry Vetrov, and
  Andrew~Gordon Wilson.
\newblock Averaging weights leads to wider optima and better generalization.
\newblock In \emph{34th Conference on Uncertainty in Artificial Intelligence},
  2018.

\bibitem[Jastrz{\k{e}}bski et~al.(2017)Jastrz{\k{e}}bski, Kenton, Arpit,
  Ballas, Fischer, Bengio, and Storkey]{threeFactors}
Stanis{\l}aw Jastrz{\k{e}}bski, Zachary Kenton, Devansh Arpit, Nicolas Ballas,
  Asja Fischer, Yoshua Bengio, and Amos Storkey.
\newblock Three factors influencing minima in sgd.
\newblock \emph{arXiv preprint arXiv:1711.04623}, 2017.

\bibitem[Jiang et~al.(2020)Jiang, Neyshabur, Mobahi, Krishnan, and
  Bengio]{jiang2020fantastic}
Yiding Jiang, Behnam Neyshabur, Hossein Mobahi, Dilip Krishnan, and Samy
  Bengio.
\newblock Fantastic generalization measures and where to find them.
\newblock In \emph{Proceedings of the International Conference on Learning
  Representations}, 2020.

\bibitem[Jones et~al.(1994)Jones, McKay, and Hu]{jones1994variable}
MC~Jones, IJ~McKay, and T-C Hu.
\newblock Variable location and scale kernel density estimation.
\newblock \emph{Annals of the Institute of Statistical Mathematics},
  46\penalty0 (3):\penalty0 521--535, 1994.

\bibitem[Keskar et~al.(2017)Keskar, Mudigere, Nocedal, Smelyanskiy, and
  Tang]{keskarLarge}
Nitish~Shirish Keskar, Dheevatsa Mudigere, Jorge Nocedal, Mikhail Smelyanskiy,
  and Ping Tak~Peter Tang.
\newblock On large-batch training for deep learning: Generalization gap and
  sharp minima.
\newblock In \emph{Proceedings of the International Conference on Learning
  Representatiosn}, 2017.

\bibitem[Krantz and Parks(2008)]{Krantz}
Steven~G. Krantz and Harold~R. Parks.
\newblock \emph{Geometric integration theory}.
\newblock Springer Science and Business Media, 2008.

\bibitem[Krizhevsky(2009)]{Krizhevsky09learningmultiple}
Alex Krizhevsky.
\newblock Learning multiple layers of features from tiny images.
\newblock Technical report, University of Toronto, 2009.

\bibitem[LeCun and Cortes(2010)]{lecun-mnisthandwrittendigit-2010}
Yann LeCun and Corinna Cortes.
\newblock {MNIST} handwritten digit database.
\newblock Technical report, AT\&T Labs, 2010.

\bibitem[LeCun et~al.(1990)LeCun, Boser, Denker, Henderson, Howard, Hubbard,
  and Jackel]{lecun1990handwritten}
Yann LeCun, Bernhard~E Boser, John~S Denker, Donnie Henderson, Richard~E
  Howard, Wayne~E Hubbard, and Lawrence~D Jackel.
\newblock Handwritten digit recognition with a back-propagation network.
\newblock In \emph{Advances in Neural Information Processing Systems}, pages
  396--404, 1990.

\bibitem[Liang et~al.(2019)Liang, Poggio, Rakhlin, and Stokes]{fisherRao}
Tengyuan Liang, Tomaso Poggio, Alexander Rakhlin, and James Stokes.
\newblock Fisher-rao metric, geometry, and complexity of neural networks.
\newblock \emph{International Conference on Artificial Intelligence and
  Statistics (AISTATS)}, 2019.

\bibitem[Nagarajan and Kolter(2019)]{nagarajan2019uniform}
Vaishnavh Nagarajan and J~Zico Kolter.
\newblock Uniform convergence may be unable to explain generalization in deep
  learning.
\newblock In \emph{Advances in Neural Information Processing Systems}, pages
  11611--11622, 2019.

\bibitem[Neyshabur et~al.(2015{\natexlab{a}})Neyshabur, Salakhutdinov, and
  Srebro]{pathSGD}
Behnam Neyshabur, Ruslan Salakhutdinov, and Nathan Srebro.
\newblock Path-sgd: Path-normalized optimization in deep neural networks.
\newblock In \emph{Advances in Neural Information Processing Systems},
  volume~28, 2015{\natexlab{a}}.

\bibitem[Neyshabur et~al.(2015{\natexlab{b}})Neyshabur, Tomioka, and
  Srebro]{indcutiveBias}
Behnam Neyshabur, Ryota Tomioka, and Nathan Srebro.
\newblock In search of the real inductive bias: {O}n the role of implicit
  regularization in deep learning.
\newblock \emph{Workshop contribution at the International Conference on
  Learning Representations}, 2015{\natexlab{b}}.

\bibitem[Neyshabur et~al.(2015{\natexlab{c}})Neyshabur, Tomioka, and
  Srebro]{neyshabur2015norm}
Behnam Neyshabur, Ryota Tomioka, and Nathan Srebro.
\newblock Norm-based capacity control in neural networks.
\newblock In \emph{Conference on Learning Theory}, pages 1376--1401,
  2015{\natexlab{c}}.

\bibitem[Neyshabur et~al.(2017)Neyshabur, Bhojanapalli, McAllester, and
  Srebro]{exploringGen}
Behnam Neyshabur, Srinadh Bhojanapalli, David McAllester, and Nati Srebro.
\newblock Exploring generalization in deep learning.
\newblock In \emph{Advances in Neural Information Processing Systems}, pages
  5947--5956, 2017.

\bibitem[Novak et~al.(2018)Novak, Bahri, Abolafia, Pennington, and
  Sohl-Dickstein]{sensitivityGeneralization}
Roman Novak, Yasaman Bahri, Daniel~A Abolafia, Jeffrey Pennington, and Jascha
  Sohl-Dickstein.
\newblock Sensitivity and generalization in neural networks: an empirical
  study.
\newblock In \emph{Proceedings of the International Conference on Learning
  Representations}, 2018.

\bibitem[Paszke et~al.(2019)Paszke, Gross, Massa, Lerer, Bradbury, Chanan,
  Killeen, Lin, Gimelshein, Antiga, Desmaison, Kopf, Yang, DeVito, Raison,
  Tejani, Chilamkurthy, Steiner, Fang, Bai, and Chintala]{NEURIPS2019_9015}
Adam Paszke, Sam Gross, Francisco Massa, Adam Lerer, James Bradbury, Gregory
  Chanan, Trevor Killeen, Zeming Lin, Natalia Gimelshein, Luca Antiga, Alban
  Desmaison, Andreas Kopf, Edward Yang, Zachary DeVito, Martin Raison, Alykhan
  Tejani, Sasank Chilamkurthy, Benoit Steiner, Lu~Fang, Junjie Bai, and Soumith
  Chintala.
\newblock Pytorch: An imperative style, high-performance deep learning library.
\newblock In \emph{Advances in Neural Information Processing Systems 32}, pages
  8024--8035, 2019.

\bibitem[Pedregosa et~al.(2011)Pedregosa, Varoquaux, Gramfort, Michel, Thirion,
  Grisel, Blondel, Prettenhofer, Weiss, Dubourg, Vanderplas, Passos,
  Cournapeau, Brucher, Perrot, and Duchesnay]{scikit-learn}
F.~Pedregosa, G.~Varoquaux, A.~Gramfort, V.~Michel, B.~Thirion, O.~Grisel,
  M.~Blondel, P.~Prettenhofer, R.~Weiss, V.~Dubourg, J.~Vanderplas, A.~Passos,
  D.~Cournapeau, M.~Brucher, M.~Perrot, and E.~Duchesnay.
\newblock Scikit-learn: Machine learning in {P}ython.
\newblock \emph{Journal of Machine Learning Research}, 12:\penalty0 2825--2830,
  2011.

\bibitem[Petzka et~al.(2019)Petzka, Adilova, Kamp, and
  Sminchisescu]{petzka2019reparameterization}
Henning Petzka, Linara Adilova, Michael Kamp, and Cristian Sminchisescu.
\newblock A reparameterization-invariant flatness measure for deep neural
  networks.
\newblock In \emph{Workshop on Science meets Engineering of Deep Learning at
  NeurIPS}, 2019.

\bibitem[Rangamani et~al.(2019)Rangamani, Nguyen, Kumar, Phan, Chin, and
  Tran]{scaleInvariantMeasure}
Akshay Rangamani, Nam~H. Nguyen, Abhishek Kumar, Dzung~T. Phan, Sang~H. Chin,
  and Trac~D. Tran.
\newblock A scale invariant flatness measure for deep network minima.
\newblock \emph{arXiv preprint arXiv:1902.02434}, 2019.

\bibitem[Sagun et~al.(2019)Sagun, Evci, Guney, Dauphin, and
  Bottou]{sagun2017empirical}
Levent Sagun, Utku Evci, V~Ugur Guney, Yann Dauphin, and Leon Bottou.
\newblock Empirical analysis of the hessian of over-parametrized neural
  networks.
\newblock \emph{Workshop contribution at the International Conference on
  Learning Representations}, 2019.

\bibitem[Silverman(1986)]{silverman1986density}
Bernard~W. Silverman.
\newblock \emph{Density estimation for statistics and data analysis}.
\newblock Monographs on Statistics and Applied Probability. Chapman and Hall,
  1986.

\bibitem[Sun et~al.(2020)Sun, Zhang, Ren, Luo, and Li]{sun2020exploring}
Xu~Sun, Zhiyuan Zhang, Xuancheng Ren, Ruixuan Luo, and Liangyou Li.
\newblock Exploring the vulnerability of deep neural networks: A study of
  parameter corruption.
\newblock \emph{arXiv preprint arXiv:2006.05620}, 2020.

\bibitem[Szegedy et~al.(2013)Szegedy, Zaremba, Sutskever, Bruna, Erhan,
  Goodfellow, and Fergus]{intriguingProperties}
Christian Szegedy, Wojciech Zaremba, Ilya Sutskever, Joan Bruna, Dumitru Erhan,
  Ian Goodfellow, and Rob Fergus.
\newblock Intriguing properties of neural networks.
\newblock \emph{arXiv preprint arXiv:1312.6199}, 2013.

\bibitem[Tikhonov et~al.(1995)Tikhonov, Goncharsky, Stepanov, and
  Yagola]{tikhonov}
Andrey~Nikolayevich Tikhonov, A.~Goncharsky, V.~V. Stepanov, and
  Anatolij~Grigorevic Yagola.
\newblock \emph{Numerical Methods for the Solution of Ill-Posed Problems}.
\newblock Mathematics and Its Applications. Springer Netherlands, 1995.

\bibitem[Tsuzuku et~al.(2020)Tsuzuku, Sato, and Sugiyama]{tsuzuku}
Yusuke Tsuzuku, Issei Sato, and Masashi Sugiyama.
\newblock Normalized flat minima: Exploring scale invariant definition of flat
  minima for neural networks using {PAC}-{B}ayesian analysis.
\newblock In \emph{Proceedings of the 37th International Conference on Machine
  Learning}, pages 9636--9647, 2020.

\bibitem[Wang et~al.(2018)Wang, Keskar, Xiong, and
  Socher]{identifyingGenProperties}
Huan Wang, Nitish~Shirish Keskar, Caiming Xiong, and Richard Socher.
\newblock Identifying generalization properties in neural networks.
\newblock \emph{arXiv preprint arXiv:1809.07402}, 2018.

\bibitem[Wu et~al.(2020)Wu, Xia, and Wang]{adversarial_weightPerturbation}
Dongxian Wu, Shu-Tao Xia, and Yisen Wang.
\newblock Adversarial weight perturbation helps robust generalization.
\newblock In \emph{Advances in Neural Information Processing Systems},
  volume~33, pages 2958--2969, 2020.

\bibitem[Xu and Mannor(2012)]{robustnessGeneralization}
Huan Xu and Shie Mannor.
\newblock Robustness and generalization.
\newblock \emph{Machine learning}, 86\penalty0 (3):\penalty0 391--423, 2012.

\bibitem[Yao et~al.(2019)Yao, Gholami, Lei, Keutzer, and
  Mahoney]{yao2018hessian}
Zhewei Yao, Amir Gholami, Qi~Lei, Kurt Keutzer, and Michael~W Mahoney.
\newblock Hessian-based analysis of large batch training and robustness to
  adversaries.
\newblock In \emph{Advances in Neural Information Processing Systems},
  volume~32, 2019.

\bibitem[Zhang et~al.(2017)Zhang, Bengio, Hardt, Recht, and
  Vinyals]{understandingRethinking}
Chiyuan Zhang, Samy Bengio, Moritz Hardt, Benjamin Recht, and Oriol Vinyals.
\newblock Understanding deep learning requires rethinking generalization.
\newblock In \emph{Proceedings of the International Conference on Learning
  Representations}, 2017.

\bibitem[Zhang et~al.(2018)Zhang, Liao, Rakhlin, Miranda, Golowich, and
  Poggio]{sgdAndFlat}
Chiyuan Zhang, Qianli Liao, Alexander Rakhlin, Brando Miranda, Noah Golowich,
  and Tomaso Poggio.
\newblock Theory of deep learning iib: Optimization properties of sgd.
\newblock \emph{arXiv preprint arXiv:1801.02254}, 2018.

\bibitem[Zheng et~al.(2020)Zheng, Zhang, and Mao]{zheng2020regularizing}
Yaowei Zheng, Richong Zhang, and Yongyi Mao.
\newblock Regularizing neural networks via adversarial model perturbation.
\newblock \emph{arXiv preprint arXiv:2010.04925}, 2020.

\end{thebibliography}

\newpage
% \documentclass{scrartcl}
% %\usepackage[left=2.cm, right=2.cm]{geometry}

% \include{preamble}

% \usepackage[numbers]{natbib}
% \usepackage{subfigure}
% \myexternaldocument{algoAgnosticGen}

% \begin{document}

\appendix

\setcounter{theorem}{6}
\setcounter{equation}{7}

\noindent
\paragraph{Organization of the Appendix}

The appendix is organized as follows:\\

\textbf{\ref{sct:related_work} -- Related work} contains an extended discussion on related work.

\textbf{\ref{appdx:locallabelChanges} -- The effect of local label changes} discusses consequences for the association of flatness and generalization for general output function $y(x)$ without the assumption of locally constant labels.

\textbf{\ref{appdx:experiments_details} -- Details on the Empirical Validation} contains a detailed description of the experiments.

\textbf{\ref{appdx:proofs} -- Proofs} contains the full proofs to all statements. In detail:

{\setlength{\parindent}{5ex}\ref{proof:prp2}: Proposition~\ref{prp:lambdaA},

\vspace{-4pt} \ref{proof:thm4}: Theorem~\ref{thm:reparameterizations},

\vspace{-4pt} \ref{proof:thm5}: Theorem~\ref{thm:averaging},
 
\vspace{-4pt} \ref{proof:thm6}: Theorem~\ref{thm:FlatnessGenBound}.}

%In addition, \ref{proof:thm5} contains a clarification of the Definition~\ref{def:featureRobustness} on feature robustness and Theorem~\ref{thm:averaging} on the assumption of locally constant labels.

\textbf{\ref{appdx:maxFlatness} -- Relative flatness for a uniform bound over general distributions on feature matrices} defines a variant of relative flatness that uniformly bounds feature robustness over all feature matrices.

\section{Related Work}\label{sct:related_work}
%\paragraph{Flatness of the Loss Curve}
It has long been observed that algorithms searching for flat minima of the loss curve lead to better generalization~\citep{flatMinima, simplifyingByFlat}. More recently, an association between flatness and low generalization error has also been validated empirically in deep learning~\citep{keskarLarge, sensitivityGeneralization, identifyingGenProperties}. Here, flatness is measured by the Hessian of the empirical loss evaluated at the model at hand.
Indeed, in their recent extensive empirical study of generalization measures, \citet{jiang2020fantastic} found that measures based on flatness have the highest correlation with generalization.

For models trained with stochastic gradient descent (SGD), this could present a (partial) explanation for their generalization performance, since the convergence of SGD can be connected to flat local minima by studying SGD as an approximation of a stochastical differential equation ~\citep{sgdAndFlat,threeFactors}. However, while large and small batch methods appear to converge in different basins of attraction, the basins can be connected by a path of low loss, i.e., they can actually converge into the same basin ~\citep{sagun2017empirical}. Moreover, as~\citet{dinhSharp} remarked, classical flatness measures---which are based only on the Hessian of the loss function---cannot theoretically be related to generalization: For deep neural networks with ReLU activation functions, there are linear reparameterizations that leave the network function unchanged (hence, also the generalization performance), but change any measure derived only from the loss Hessian. 
Novel measures related to flatness have been proposed that are invariant to linear reparameterizations~\cite{tsuzuku,scaleInvariantMeasure,fisherRao}. \citet{scaleInvariantMeasure} measure flatness in the quotient space of a suitable equivalence relation, and~\citet{fisherRao} utilize the Fisher-Rao metric, but the theoretical connection of these two measures to generalization is not well-understood. \citet{exploringGen} noted that the reparamterization-issue can in general be resolved by balancing a measure of flatness with a norm on the parameters, which is the way that normalized flatness~\citep{tsuzuku}, Fisher-Rao metric~\citep{fisherRao} and our proposed relative flatness become reparameterization-invariant. However, the solution proposed in~\citet{exploringGen} necessitates data-dependent priors~\citep{dziugaite2018data}
or related approaches, which "adds non-trivial costs to
the generalization bounds"~\citep{tsuzuku}.

%\TODO{Decide for one of the two following paragraphs.}
The question arises in which way the loss Hessian and parameter norm should be combined. A simple scaling of the full Hessian with the squared parameter norm does not provide a reparamterization-invariant measure. Doing so for each layer independently and summing up the results provides a measure that is only invariant under layer-wise reparameterizations. Similarly, only considering a single feature layer yields a measure that is layer-wise reparameterization invariant~\citep{petzka2019reparameterization}.
While the resulting measure can also be analyzed within our framework to obtain a bound on feature robustness, our proposed measure yields a tighter bound and is also invariant under neuron-wise reparameterizations.
%%%%%%%%%%%%%%%%%%%%%%%%%

%The question arises in which way the loss Hessian and parameter norm should be combined. A simple scaling of the full Hessian with the squared parameter norm does not provide a reparamterization-invariant measure. Scaling the Hessian for each layer by its squared parameter norm provides a layer-wise reparameterization invariant flatness measure. This measure can be analyzed in our theoretical framework and which also yields a bound on feature robustness. Our proposed measure instead is only evaluated on a single feature layer and is neuron- and layer-wise reparameterization invariant, including reparameterizations within the feature layer.  

\citet{tsuzuku} derive a flatness measure that scales an approximation to the loss Hessian by a parameter-dependent term. % to achieve invariance under linear reparameterizations.
%Relative flatness is the most similar to a measure proposed by \citet{tsuzuku},which is also invariant under all linear reparameterizations.
Their proposed measure correlates well with generalization and is theoretically connected to it via the PAC-Bayesian framework. However, this connection requires the assumption of Gaussian priors and posteriors and is not informative with respect to conditions under which this connection holds. Moreover the measure is impractical, since computing it requires solving an optimization problem for every layer that can be numerically unstable. (\citet{tsuzuku} propose a solution to the numerical instability at the cost of losing the reparameterization-invariance.) 
Instead, relative flatness can be computed directly and takes only parameters of a specific layer into account---although combining relative flatness of all layers by simple summation is possible.

% Their proposed measure requires the solution to an optimization problem and takes only (approximations to) diagonal entries of the loss Hessians into account, while relative flatness is a natural measure of the full loss Hessian of a specified layer. Note further that the measure proposed by \citet{tsuzuku} can be numerically instable; the authors propose how to make the measure numerically stable but then it looses its reparameterization-invariance. \citet{tsuzuku} further provide a theoretical connection between flatness and generalization within the PAC-Baysian framework under the assumption of Gaussian priors and posteriors, but their approach does not uncover conditions when flatness is associated with generalization. In difference to all related measures, relative flatness takes only the parameters of a specific layer into account. Naturally, relative flatness can be measured in all layers and each layer-wise measure can be subsequently combined.

A series of recent papers studies flatness by minimizing the loss at local perturbations of the parameters considering $\min_\mathbf{a} \Ecal_{emp}(f(\w+\mathbf{a}),S)$ \citep{foret2021sharpnessaware, zheng2020regularizing, sun2020exploring,adversarial_weightPerturbation}. Regularization techniques enforcing these notions of flatness during training in classification tasks lead to better generalization. These empirical results follow earlier works by \citet{entropySGD} and \citet{izmailov2018averaging} that similarly obtained better generalization by enforcing flatter minima. Their observations are well-explained by our theory: Low error at perturbations $\Ecal_{emp}(f(\w+\mathbf{a}),S)$ lead to good generalization around training samples. This requires that the underlying distribution has (approximately) locally constant labels (using key equation \eqref{eq:keyEquation}), which is reasonable for the image classification tasks they consider.

\citet{robustnessGeneralization} propose a notion of robustness over a partion of the input space and derive generalization bounds based on it. However, their notion requires the choice of a partitioning of the input space before seeing any samples. Thus, robustness over the partition can be hard to estimate for a model that depends on a sample set $S$. Our notion of feature robustness is measured around a given sample set and thus does not require a uniform data-independent partitioning. 
%differs from theirs in that our robustness is measured around a given sample set, while their notion requires the choice of a partition of the input space before seeing samples. 
Such a sample-dependent notion of robustness is necessary to connect it to the flatness of the loss surface, since flatness is a local property around training points. 

\citet{sensitivityGeneralization} find that robustness to input perturbation as measured by the input-output Jacobian correlates well with generalization on classification tasks. This is in line with our findings applied to $\phi=id_\Xcal$ chosen as the identity (for neural networks this means considering the input layer as features): it follows from Equation~\ref{eq:keyEquation} that robustness to input perturbations directly relates to flatness. Therefore, these findings give additional empirical evidence to the correlation between flatness and generalization. \citet{yao2018hessian} study the Hessian with respect to the input $x\in\Xcal$ and also find that robust learning tends to converge to minima where the input-ouptut Hessian has small eigenvalues.

\newpage

\section{The effect of local label changes}
\label{appdx:locallabelChanges}
For classification tasks with one-hot vectors as labels, the assumption of locally constant labels, i.e., locally constant target output function $y(x)$, seems reasonable since we would not expect the class label to change under (infinitesimally) small changes. One could nonetheless consider a smooth output function with values encoding class probabilities for classification, which may change locally around the training points. For regression tasks, the assumption of locally constant output function is rather unrealistic or at the very least restrictive.

Taking the term defining feature robustness \eqref{eq:defRobustness} as a starting point, we investigate its connection to flatness when the output function $y(x)$ is a smooth function. In the usual setting of machine learning, this information is unknown. We will show that label changes can contribute stronger to the loss in neighborhoods around training samples than (relative) flatness.

%$\epsilon$-interpolation as in \eqref{eq:interpolationLableChanges} assumes the knowledge of optimal labels $y(x)$ for each $x$. Feature robustness now measures whether the prediction deviates from the changing optimal label when features are perturbed. 
To investigate the label dependence, we use the same trick as in \eqref{eq:InputToParameters} to transfer perturbations in the input $x$ to perturbations in parameter space $\w$. To simplify the analysis, we apply feature robustness to the input space (i.e., we only consider $\phi=id_\Xcal$ here). Let $f(x,\w)=\psi(\w x)$ be a model composed of a matrix multiplication of $x$ with $\w$ and a differentiable predictor function $\psi$.
\begin{eqnarray*}%\label{eq:featureRobustnessLabelChanges}
%\begin{split}
\Ecal_{\tilde{F}}(f,S,A)&= \frac 1n \sum_{i=1}^n \left ( \loss(f(x_i+\delta Ax_i,\w),y[x_i+\delta Ax_i]) - \loss(f(x_i,\w),y_i)\right )\\
&= \frac 1n \sum_{i=1}^n \left ( \loss(f(x_i, \w+\delta \w A),y[x_i+\delta Ax_i]) - \loss(f(x_i,\w),y_i)\right )
%\end{split}
\end{eqnarray*}

Defining a function \begin{equation} \gamma_i(\delta) =  \loss(f(x_i, \w+\delta \w A),y[x_i+\delta Ax_i]), \end{equation} we have that $\Ecal_{\tilde{F}}(f,S,A)= \frac 1n \sum_{i=1}^n \gamma_i(\delta)$. For each $\gamma_i$ we use Taylor approximation in $\delta$. In the following, we write $\loss_\w(x_i,\w^*,y_i)$ for the first derivative of the loss with changes in $\w$ at $x_i,y_i=y[x_i]$ and $\w^*$, and we write  $\loss_{y}(x_i,\w^*,y_i)$ for the first derivative of the loss with changes of the output $y$ at $x_i,y_i=y[x_i]$ and $\w^*$. Similarly, we consider second derivatives $\loss_{\w\w},\loss_{yy}$ and $\loss_{\w y}$. Finally, we denote the derivative of $y(x)$ with respect to $x$ by $y_x$ and the second derivative by $y_{xx}$. Then,
\begin{equation}
\gamma_i'(0)=\loss_\w(x_i,\w^*,y_i)\cdot(\w^*A) + \loss_y(x_i,\w^*,y_i)\cdot (y_x(x_i)\cdot Ax_i)
\end{equation}
and 
\begin{eqnarray}\label{eq:nonconstantLabels}
%\begin{split}
\gamma_i''(0)=(\w^*A)^T \loss_{\w\w}(x_i,\w^*,y_i)(\w^*A) + (y_x(x_i)\cdot Ax_i)^T \loss_{yy}(x_i,\w^*,y_i)(y_x(x_i)\cdot Ax_i)\\ 
+ \sum_{\text{labels }c} \loss_{y_c}(x_i,\w^*,y_i)\cdot (Ax_i)^T (y_c)_{xx}(x_i)\cdot (Ax_i) +2 (y_x(x_i)\cdot Ax_i)^T \loss_{y\w}(x_i,\w^*,y_i)(\w^*A) 
%\end{split}
\end{eqnarray}
At a critical point we have that $\sum_i \loss_\w(x_i,\w^*,y_i)=0$, but since we do not know how the target output function $y(x)$ changes locally, we do not necessarily\footnote{This depends on the loss function in use.} enforce that $\sum_i \loss_y(x_i,w^*,y_i)=0$ at a local optimum. In that case, $\Ecal_{\tilde{F}}(f,S,A)=\sum_i \loss_y(x_i,w^*,y_i)\delta +\mathcal{O}(\delta^2)$ has a non-zero term of first order in $\delta$ and flatness only contributes as a term of order two.  Similarly, other terms in \eqref{eq:nonconstantLabels} can be nonzero, further reducing the influence of relative flatness to a bound on feature robustness. 

As an interesting special case, we note that 
%Then, the first order changes of the label is the main contribution to the local loss. However, considering averages of feature matrices as before, averages the contribution to zero. 
for one-hot encoded labels in classification and letting the output function $y(x)$ describe a parameter vector of a conditional label-distribution given $x$, we have $y_x(x_i)=0$ (recall that we suppose $y(x_i)=y_i$) as each vector component is either $1$ or $0$ and must be a local extreme point ($y(x)$ cannot contain values larger than $1$ or smaller than $0$ by assumption),

We leave a detailed investigation of the consequences of label changes as future work, but identify the implicit assumption of locally constant labels in loss Hessian-based flatness measures as a possible limitation: Flatness can only be descriptive if optimal label changes are approximately locally constant. The fact that a strong correlation between flatness and generalization gap has been often observed points to the fact that distributions in practice satisfy this implicit assumption.

%%%%%%%%%%%%%%%%%%%%%%%%%%%%%%%%%%%%%%%%%%%%%%%%%%
\newpage
\section{Details on the Empirical Validation}
\label{appdx:experiments_details}
Here we provide additional details on the empirical evaluation. Jupyter notebooks containing the experiments are available at~\url{\experimentURL}, ensuring reproducibility, together with an implementation of the relative flatness measure in pytorch~\citep{NEURIPS2019_9015}. 
%Do we need to specify the hardware?

\subsection{Synthetic Experiments} The experiments on \emph{locally constant labels} and \emph{approximating representativeness} use a synthetic sample in feature space. The schema for both experiments is to
\begin{enumerate}
    \item create a synthetic dataset in feature space $S^\phi$ and test set $T^\phi$,
    \item create a model $f=\phi\circ\psi$,
    \item derive input data as $S=\left(\phi^{-1}\left(S^\phi_x\right),S^\phi_y\right)$, $T=\left(\phi^{-1}\left(T^\phi_x\right),T^\phi_y\right)$
    \item compute relative flatness (or other measures) of $f$ on $S$,
    \item and estimate its generalization gap by computing the empirical risk of $f$ on $S$, and computing the test error on the test test $T$ to estimate the risk.
\end{enumerate}

1) To create $S^\phi$ with a given class separation $c$, we randomly sample 4 cluster centroids $\theta$ from a hypercube in $\RR^6$ and scale them so that their distance is $c$. We then sample a random covariance matrix $\Sigma$ for each cluster and sample points from a Gaussian $\mathcal{N}(\theta,\Sigma)$. Furthermore, we create two redundant features that are a random linear combination of the $6$ informative features. We obtain labels by assigning two clusters to class $1$ and the other two to class $-1$. 

2) We create the model $f$ by first training a linear model $\psi$ on $S^\phi$ using ridge regression from scikit-learn~\citep{scikit-learn}. We then sample a random 4-layer MLP (with architecture 784-512-128-16-8, tanh activation, and Glorot initialization~\citep{glorot2010understanding}) that we use as feature extractor $\phi$. With this, we obtain the $5$-layer MLP $f=\phi\circ\psi$ by adding an $8-2$ layer with weights obtained from $\psi$. 

3) We obtain input data $S$ by reverse propagation of samples in feature space $S^\phi_x$ through the $4$-layer MLP $\phi$. This is an approximation to the inverse feature extractor $\phi^{-1}$. For the output of each layer $z$, we first compute $z'=tanh^{-1}(z)$, i.e., the inverse of the activation function. We then solve $Wz + b = x$, where $W,b$ are the weights and bias of that layer, and $x$ is the corresponding input we want to compute. This yields $S_x=\phi^{-1}(S^\phi_x)$. Note that this reverse propagation of samples introduces a small error. To keep experiments realistic, we discard $S^\phi$ after this step and use only the input dataset $S$ and model $f$ in our computations. 

4) We compute relative flatness as in Def.~\ref{def:flatnessMeasure} (an implementation in pytorch is available on github, see above). %We use $\delta=0.1$, and the feature space has dimension $m=8$ (see point 1). 

5) We compute the empirical risk of $f$ on $S$ and estimate the risk on a test set. For the experiments on locally constant labels, generate $5000$ samples, use a training set of size $500$, a test set of size $4500$ (to ensure an accurate estimate of the risk), and repreat the experiment $100$ times for each class separation $c$. For the experiment on approximating representativeness, we use a sample of size $600$ and perform $3$-fold cross-validation. 

\paragraph{Locally constant labels:} For classification, labels are locally constant if in a neighborhood around each point the label does not change. They are approximately locally constant, if this holds for most points. By increasing the distance between the means of the Gaussians, we decrease the likelihood of a point within a neighborhood having a different label. For a finite sample, this means that the likelihood of observing two points close by with different labels decreases. Thus, by increasing the class separation parameter, we increase the degree of locally constant labels.

\paragraph{Approximating representativeness:} A finite random sample as described in 1) has a higher chance of being representative when the means of the Gaussians have a high distance, because each individual Gaussian can be interpolated easily. Of course, the actual representativeness of a sample at hand can vary. Note that this is a very simple form of generating datasets with varying "difficulty". It will be interesting to further explore the impact of the choice of data distribution on (an approximation to) representativeness.

Experiments on the synthetic datasets are run on a laptop with Intel Core i7 and NVIDIA GeForce GTX 965 M 2 GB GPU. The code of the experiments is provided as a jupyter notebook so that they can be easily reproduced.

\subsection{Relative Flatness Correlates with Generalization} 

\begin{figure}
\centering
\begin{minipage}{0.48\textwidth}
  \centering
  \includegraphics[width=\linewidth]{images/experiments/flatness/comparison.png}
  \caption{Generalization gap and various flatness measures for $110$ local minima as presented in Fig.~\ref{fig:relativeFlatness}. The generalization gap correlates stronger with relative flatness than standard flatness, Fisher-Rao norm, a PAC-Bayes based measure and the weights norm.}
  \label{fig:relativeFlatnessAppdx}
\end{minipage}%
\hspace{1em}
\begin{minipage}{0.48\textwidth}
  \centering
  \includegraphics[width=\linewidth]{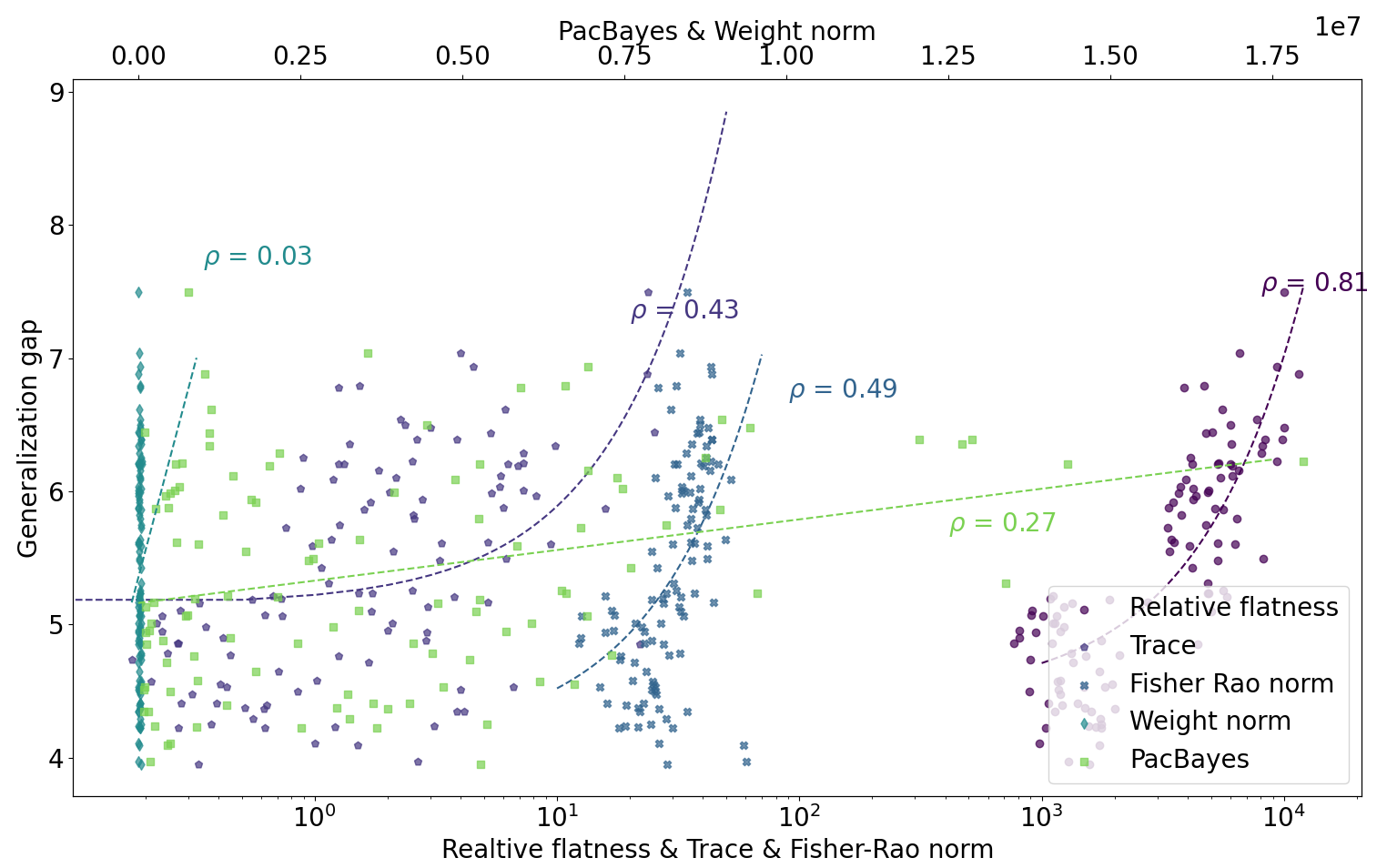}
  \caption{Modifying the local minima in the plot Fig.~\ref{fig:relativeFlatnessAppdx} by reparameterization shows that the proposed relative flatness and the Fisher-Rao norm are invariant to them. It furthermore shows a strong decline in correlation for all other measures.}
  \label{fig:relativeFlatnessAndReparam}
\end{minipage}
\end{figure}

In this experiment, we validate that the proposed relative flatness correlates strongly with generalization in practice. For that, we measure relative flatness (as well as classical flatness measured by the trace of the loss Hessian, the Fisher-Rao norm, a PAC-Bayes based measure~\footnote{The implementation of PAC-Bayes based flatness measure is taken from \url{https://github.com/nitarshan/robust-generalization-measures/blob/master/data/generation/measures.py}}, and the weight norm) together with the generalization gap for various local minima. 

\begin{wrapfigure}{r}{9cm}
        \centering
        \includegraphics[width=9cm]{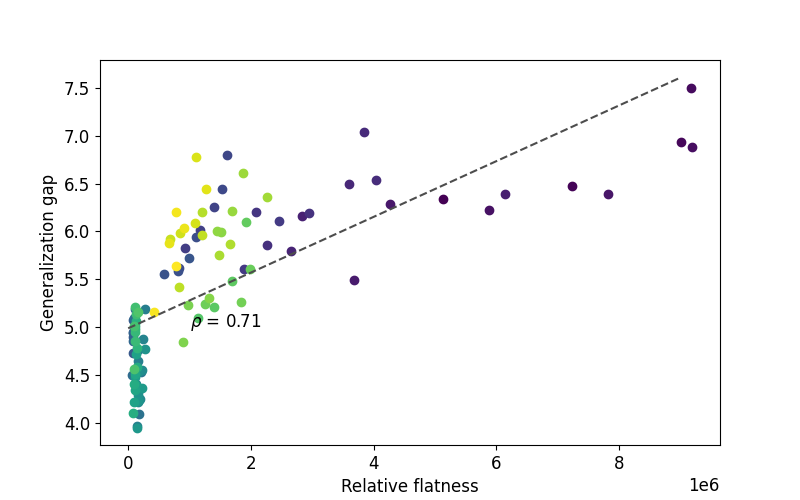}
        \caption{The generalization gap for various local minima correlates with relative flatness measured on the layer different from penultimate layer.}
        \label{fig:relativeFlatness_next_layer}
\end{wrapfigure}

To obtain model parameters at various local minima, we train networks (LeNet5~\citep{lecun1990handwritten}) on the CIFAR10 dataset until convergence (measured in terms of achieving a loss of less than $0.1$ during an epoch, which has been used as a criteria for convergence in similar experiments~\citep{jiang2020fantastic}) with varying hyperparameters. In accordance to works studying the impact of hyperparameters on generalization~\citep{scaleInvariantMeasure, keskarLarge, sensitivityGeneralization, identifyingGenProperties, jiang2020fantastic}, we vary learning rate, mini-batch size, initialization, and optimizer. We vary the mini batch size in ${64, 128, 256, 512, 1024}$, and the learning rate in ${0.0001, 0.02, 0.05}$, running $10$ randomly initialized training rounds for each setup. We use SGD, ADAM, and RMSProp as optimizers. We only use combinations that lead to convergence. The experiments were conducted on a cluster node with 4 NVIDIA GPU GM200 (GeForce GTX TITAN X). As discussed in Sec.~\ref{sec:experiments}, relative flatness has the highest correlation with generalization from all measures we analyzed.

To study the effect of reparameterization, we apply layer-wise reparameterizations on the trained network using random factors in the interval $[5,25]$ which yields a set of novel local minima. The results in Fig.~\ref{fig:relativeFlatnessAndReparam} show that both our proposed relative flatness and the Fisher-Rao norm are invariant to these reparameterization. For all other measures, the correlation with generalization declines substantially. The same would hold for neuron-wise reparameterizations, since both relative flatness and the Fisher-Rao norm are also neuron-wise reparameterization invariant. 
Relative flatness and the Fisher-Rao norm are also invariant under neuron-wise reparameterizations, which could be used to further break the correlation for the other measures. For future work it would be interesting to investigate further symmetries in neural networks and the impact of reparameterizations along these symmetries on flatness measures.

In addition to the calculation of the relative flatness using the feature space of the penultimate layer, we also performed calculations for another fully-connected layer in the network. The resulting correlation can be seen in Fig.~\ref{fig:relativeFlatness_next_layer}. It keeps the high correlation value, but due to less optimal feature space we observe smaller number, than in the previous calculation. Nevertheless, it demonstrates that any $\phi$-$\psi$ separation allows to compute relative flatness.

\newpage

\section{Proofs}
\label{appdx:proofs}
\subsection{Proof of Proposition~\ref{prp:lambdaA}}\label{proof:prp2}
\begin{proof}
Let $\mathcal{K}_{h}$ denote probability distribution defined by a rotational-invariant kernel $k_h$ as in \eqref{eq:kernel} with $k_h(0,z)=\frac{1}{h^m}\cdot k \left (\frac{||z||}{h} \right)  \cdot \mathds{1}_{||z||<h}$ and let $\lambda_i(z)=k_{\delta||\phi(x_i)||}(0,z)$. Let $\mathcal{L}$ denote a continuous function on $\RR^m$ and $\mathcal{O}_m$ the set of orthogonal matrices in $\RR^{m\times m}$. We show that there exists a probability measure $\kappa$ on a set $M_\delta$ of matrices of norm smaller than $\delta$, defining a probability distribution $\mathcal{A}_\delta$, and a probability measure $\omega$ on the product space $(0,\delta]\times \mathcal{O}_m$ such that for each $z\in \RR^m\setminus\{0\}$:
\begin{equation}\label{eq:AcalToKdelta}
 \Exp{A\sim \mathcal{A}_\delta}{\mathcal{L}(z+Az)}= \Exp{(r,O)\sim \omega}{\mathcal{L}(z+rOz)} =    \Exp{\zeta\sim \mathcal{K}_{\delta||z||}}{\mathcal{L}(z+\zeta)} 
\end{equation}
 
Applying this result for each $i=1,\ldots,|S|$ to $\mathcal{L}_i(z)=\loss(\psi(\w,z) ,y_i[z])$ at $z=\phi(x_i)$ completes the proof. For all the standard measure-theoretic concepts used in the proof, we refer the reader to \cite{Krantz}.

Fix some $\zeta_0$ in $\RR^m$ with $||\zeta_0||=1$. We consider the Haar measure $\mu$ on the set of orthogonal matrices $\mathcal{O}_m$. By  \citep[Proposition 3.2.1]{Krantz} and the change of variables formula, 
we have for each $r\in(0,\delta]$
%\begin{equation*}
%    \int_{A\in \mathcal{O}_m}{L(z+A\zeta_0)} d\mu(A) =\frac{1}{\text{Vol}(S^{m-1})} \int_{\xi \in S^{m-1}} L(z+\xi) dS^{m-1}
%\end{equation*}
%where $S^{m-1}$ is the $m-1$-sphere. Hence, for each $r\in(0,\delta]$ we hence have 
\begin{equation*}
    \int_{O\in \mathcal{O}_m}{\mathcal{L}(z+r||z||O\zeta_0)}\ d\mu(O) =\frac{1}{\text{Vol}( S^{m-1})} \int_{\xi\in S^{m-1}} \mathcal{L}(z+r||z||\xi )\ d\xi
\end{equation*}
%Let $\rho_{\sigma,\delta}(r)$ denote the probability distribution of the truncated normal distribution with constant variance $\sigma$, \[\mathcal{N}_{\sigma, \delta}(U)=\int_{z \in U} \rho_{\sigma,\delta} (||z||)\ dz.\] #
where $S^{m-1}$ is the $(m-1)$-sphere. We multiply both sides by $\frac{\text{Vol}(S^{m-1})}{\delta^m} k\left(\frac{r}{\delta}\right)  r^{m-1}$, integrate over $r\in (0,\delta]$ to obtain 
\begin{equation*}
\begin{split}
%\int_{r=0}^\delta & \int_{O\in \mathcal{O}_m} {L(z+r||z||O\zeta_0)} k\left(\frac{r}{\delta}\right) r^{m-1} dr  d\mu(O)  \\
%&= \frac{1}{\text{Vol}(S^{m-1})} \int_{r=0}^\delta\int_{\xi\in S^{m-1}} L(z+r||z||\xi) k\left(\frac{r}{\delta}\right) r^{m-1}\ dr\ d\xi \\    
%&= \frac{1}{\text{Vol}(S^{m-1})} \int_{||\zeta||\leq \delta} L(z+||z||\zeta) k\left (\frac{||\zeta||}{\delta}\right ) \ d\zeta\\  
%&= \frac{1}{\text{Vol}(S^{m-1})} \int_{||\zeta||\leq \delta||z||} L(z+\zeta) \frac{1}{(||z||)^m}  k\left (\frac{||\zeta||}{\delta ||z||}\right ) \ d\zeta\\  
%&= \frac{\delta^m}{\text{Vol}(S^{m-1})} \int_{||\zeta||\leq \delta||z||} L(z+\zeta) \frac{1}{(\delta||z||)^m}  k\left (\frac{||\zeta||}{\delta ||z||}\right ) \ d\zeta\\
\frac{\text{Vol}(S^{m-1})}{\delta^m} \int_{r=0}^\delta & \int_{O\in \mathcal{O}_m} {\mathcal{L}(z+r||z||O\zeta_0)} k\left(\frac{r}{\delta}\right) r^{m-1} dr  d\mu(O)  \\
&= \frac{1}{\delta^m} \int_{r=0}^\delta\int_{\xi\in S^{m-1}} \mathcal{L}(z+r||z||\xi) k\left(\frac{r}{\delta}\right) r^{m-1}\ dr\ d\xi \\    
&= \frac{1}{\delta^m} \int_{||\zeta||\leq \delta} \mathcal{L}(z+||z||\zeta) k\left (\frac{||\zeta||}{\delta}\right ) \ d\zeta\\  
&= \int_{||\zeta||\leq \delta||z||} \mathcal{L}(z+\zeta) \frac{1}{(\delta||z||)^m}  k\left (\frac{||\zeta||}{\delta ||z||}\right ) \ d\zeta\\
%&=c\cdot \Exp{\zeta\sim \mathcal{K}_{\delta||z||}}{L(z+\zeta)}
\end{split}
\end{equation*}
%with some normalizing constant $c$. 
Introducing the product measure $\omega := \frac{\text{Vol}(S^{m-1})}{\delta^m}\cdot (k\left(\frac{r}{\delta}\right) r^{m-1} dr \times \mu )$ on $(0,\delta]\times \mathcal{O}_n$, this implies that
\begin{equation}\label{eq:measureEquality}
\Exp{(r,O)\sim \omega}{\mathcal{L}(z+r||z||O\xi_0)} =  \Exp{\zeta\sim \mathcal{K}_{\delta||z||}}{\mathcal{L}(z+\zeta)} 
\end{equation}
The measure $\omega$ can be pushed forward to a measure on matrices of norm $||A||\leq \delta$. For this, consider the homeomorphism 
\[H: (0,\delta]\times \mathcal{O}_n \rightarrow \{ rO\ |\ r\in (0,\delta], O\in \mathcal{O}_n\}=:M_\delta\subseteq \{A\in \RR^{n\times n}\ |\ ||A||\leq \delta\} \]
given by $H(r,O)=rO$. We use the inverse of $H$ to push forward the measure $\omega$  to a measure $\kappa$ on $M_\delta$ and obtain from \eqref{eq:measureEquality} that
\[ \Exp{A\sim (M_\delta,\kappa)}{\mathcal{L}(z+||z||A\zeta_0)} = \Exp{\zeta\sim \mathcal{K}_{\delta||z||}}{\mathcal{L}(z+\zeta)}  \]
Finally, there exists an orthogonal matrix $O$ such that $O||z||\zeta_0=z$. Since $\kappa(A)=\kappa(AO^{-1})$ by definition of $\kappa$ and since $M_\delta O=M_\delta$, we get for any $z$ that 
\begin{equation*}
\begin{split}
 \Exp{\zeta\sim \mathcal{K}_{\delta||z||}}{\mathcal{L}(z+\zeta)} &= \Exp{A\sim (M_\delta,\kappa)}{\mathcal{L}(z+A||z||\zeta_0)} \\
&= \Exp{A\sim (M_\delta O^{-1},\kappa)}{\mathcal{L}(z+AO||z||\zeta_0)} \\
&= \Exp{A\sim (M_\delta,\kappa)}{\mathcal{L}(z+Az)} \\
\end{split}
\end{equation*}
Hence, the probability distribution $\mathcal{A}_\delta$ on  matrices  with norm bounded by $\delta$ defined by the probability measure $\kappa$ with support on $M_\delta$, and the space $(0,\delta]\times \mathcal{O}_m$ equipped with $\omega = \frac{\text{Vol}(S^{m-1})}{\delta^m}\cdot (k\left(\frac{r}{\delta}\right) r^{m-1} dr \times \mu )$ give the desired probability distributions satisfying \eqref{eq:AcalToKdelta}.   
\end{proof}

\newpage
\subsection{Proof of Theorem~\ref{thm:reparameterizations}}\label{proof:thm4}

We rephrase Theorem~\ref{thm:reparameterizations} split into Theorem~\ref{thm:invarianceAll} and a subsequent corollary that specify the reparameterizations under consideration. Let $f=f(\w^1,b^1,\w^2,b^2,\ldots, \w^L,b^L)$ denote a ReLU network function parameterized by parameters $\w^k=w^k_{s,t}$ and bias $b^k=b^k_s$ of the $k$-th layer given by
\begin{equation*}f(x)=\w^L \sigma (\ldots \sigma(\w^l\ \sigma(\w^{l-1}\sigma(\ldots  \sigma(\w^1 x+b^1))\ldots)+b^{l-1})+b^l)\ldots )+b^L.\end{equation*}
Recall that we let $\phi^l(x)$ denote the composition of the first $(l-1)$ layers so that we obtain a decomposition $f(x,\w^l)=g^l(\w^l \phi^l(x))$ of the network. Using \eqref{eq:defFlatness} we obtain a relative flatness measure $\kappa^l_{Tr}(\w)$ for the chosen layer.

A layer-wise reparameterizaton multiplies all weights in a layer $l$ with a positive number $\lambda$ and divides the weights of another layer $l'\neq l$ by the same $\lambda$. Due to the positive homogeneity of the ReLU activation, this does not change the network function. By a neuron-wise reparameterization, we mean the operation that multiplies all weights into a neuron by some positive $\lambda$ and divides all outgoing weights of the same neuron by $\lambda$. Again, the positive homogeneity of the activation function implies that this operation does not change  the network function. A layer-wise reparameterization is simply the parallel application of neuron-wise reparameterization for all neurons of one layer with the same reparameterization parameter $\lambda>0$.

\begin{theorem}\label{thm:invarianceAll}
Let $f=f(\w^1,b^1,\w^2,b^2,\ldots, \w^L,b^L)$ denote a neural network function parameterized by parameters $\w^k=w^k_{s,t}$ and bias $b^k=b^k_s$ of the $k$-th layer. Suppose there are positive numbers $\lambda_{s,t}^k$ such that the parameters $\w_\lambda^k,b^k_\lambda$, obtained from multiplying $w_{s,t}^k$ at matrix position $(s,t)$ in layer $k$ by $\lambda_{s,t}^k$ and $b^k_s$ by $\lambda_{(s,0)}^k$, satisfy that $f(\w^1,b^1,\w^2,b^2,\ldots, \w^L,b^L)= f(\w_\lambda^1,b_\lambda^1,\w_\lambda^2,b_\lambda^2,\ldots, \w_\lambda^{L}, b_\lambda^L)$. If for the layer with index $l$ it holds that $\lambda^l_{(s,t)}=\lambda^l_{(s,t')}$ for each $s,t$ and $t'$, then %$\kappa^l(\w)=\kappa^l(\w_\lambda)$ and 
$\kappa^l_{Tr}(\w) =\kappa^l_{Tr}(\w_\lambda)$ 
for the notion of relative flatness from Definition~\ref{def:flatnessMeasure}.
%for the notions of relative flatness from Definitions~\ref{def:flatnessMeasure} and~\ref{def:flatnessMeasureAppendix}.
\end{theorem}

\begin{corollary}
Let $\sigma_i$ denote the variance of the i-th coordinate of $\phi^l(x)$ over 
samples $x\in S$ and $V=diag\left ( {\sigma_1},\ldots, {\sigma_{n_{l-1}}}\right )$. If the relative flatness measure $\kappa^l_{Tr}$ is applied to the representation %\[f=\psi(\w^lV,b^l,\w^{l+1},b^{l+1}\ldots, \w^L,b^L) \circ  \phi(V^{-1}\w^{l-1},V^{-1}b^{l-1},\ldots,\w^1,b^1),\]
$f=f(\w^1,b^1,\ldots,V^{-1}\w^{l-1},V^{-1}b^{l-1},\w^lV,b^l,\w^{l+1},b^{l+1}\ldots, \w^L,b^L)$, i.e.,
\begin{equation*}f(x)=\w^L \sigma (\ldots \sigma(\w^lV\ \sigma(V^{-1}\w^{l-1}\sigma(\ldots  \sigma(\w^1 x+b^1))\ldots)+V^{-1}b^{l-1})+b^l)\ldots )+b^L,\end{equation*}
then $\kappa^l_{Tr}$ is invariant under all neuron-wise (and layer-wise) reparameterizations 
%Let $\sigma_i$ denote the variance of the i-th coordinate of $\phi(x;\w^{l-1},b^{l-1},\ldots,\w^1,b^1)$ over 
%samples $x\in S$ and $V=diag\left ( {\sigma_1},\ldots, {\sigma_{n_{l-1}}}\right )$. If the relative flatness measures $\kappa^l$ and $\kappa^l_{Tr}$ are applied to %\[f=\psi(\w^lV,b^l,\w^{l+1},b^{l+1}\ldots, \w^L,b^L) \circ  \phi(V^{-1}\w^{l-1},V^{-1}b^{l-1},\ldots,\w^L,b^L),\]
%then $\kappa^l$ and $\kappa^l_{Tr}$ are invariant under all neuron-wise (and layer-wise) reparameterizations 
\end{corollary}

%\begin{theorem}\label{thm:invarianceAll}
%	Let $f=f(\w^1,\w^2,\ldots, \w^L)$ denote a neural network function parameterized by parameters $\w^l$ of the $l$-th layer. Suppose there are positive numbers $\lambda_{s,t}^1,\ldots,\lambda_{s,t}^L$ such that the parameters $\w_\lambda^l$ obtained from multiplying $\text{w}_{s,t}^l$ at matrix position $(s,t)$ in layer $l$ by $\lambda_{s,t}^l$ satisfy that $f(\w^1,\w^2,\ldots, \w^L)= f(\w_\lambda^1,\w_\lambda^2,\ldots, \w_\lambda^L)$ for all $\w^l$. Then  $\kappa^l(\w)=\kappa^l(\w_\lambda)$ and $\kappa^l_{Tr}(\w) =\kappa^l_{Tr}(\w_\lambda)$ for all $l$.
%\end{theorem}

\begin{proof}
We are given a neural network function $f(x; \w^1,b^1,\ldots, \w^L,b^L)$ parameterized by  parameters $\w^k$ and bias terms $b^k$ of the $k$-th layer and positive numbers $\lambda_{(s,t)}^1,\ldots,\lambda_{(s,t)}^L$ such that the parameters $\w_\lambda^k$ obtained from multiplying weight $\text{w}_{(s,t)}^k$ at matrix position $(s,t)$ in layer $k$ by $\lambda_{(s,t)}^k$ and $b_s^k$ by $\lambda_{(s,0)}^k$ satisfies that \[f(x; \w^1,b^1,\w^2,b^2,\ldots, \w^L,b^L)= f(x; \w_\lambda^1,b_\lambda^1,\w_\lambda^2,b_\lambda^2\ldots, \w_\lambda^L,b_\lambda^L)\] for all $\w^k,b^k$ and all $x$. 
%We denote by $\w_l(j)^\lambda$ the product obtained from multiplying weight $\w_l(j)_i=w^{(i,j)}_l$ at matrix position $(i,j)$ in layer $l$ by $\lambda^{(i,j)}$.
%The proof is very similar to the proof of Theorem~\ref{thm:invariance}, only this time we have to take the different parameters $\lambda_l^{(i,j)}$ into account. 

For fixed layer $l$, we denote the $s$-th row of $\w^l$ by $\w^l_s$ before reparameterization, and we denote the $s$-th row of $\w^l_\lambda$ by $\w^l_{\lambda s}$ after reparameterization. For simplicity of the notation, we will collect all bias terms in terms $\mathbf{b},\mathbf{b}_\lambda$ before and after reparameterization respectively.
%Similarly, the bias terms in $b^l_\lambda$ after reparameterization are denoted by $b^l_{\lambda s}$. 
Let 
\begin{equation*}
	\begin{split}
		F(\uu):=\sum_{i=1}^{|S|} \loss(f(x_i; \w^1,\w^2,&\ldots,[\w^l_1,\ldots,\w^l_{s-1},\uu,\w^l_{s+1},\ldots \w^l_{d}],\ldots,\w^L,\mathbf{b}),y_i)
	\end{split}
\end{equation*}
denote the loss as a function on the parameters of the $s$-th neuron in the $l$-th layer (encoded in the $s$-th row of $\w^l$) before reparameterization and 
\begin{equation*}
	\begin{split}
		\tilde F(\uu):=\sum_{i=1}^{|S|} \loss( f(x_i;  \w^1_{\lambda},\w^2_{\lambda} ,&\ldots,[\w^l_{\lambda 1} ,\ldots,\w^l_{\lambda(s-1)} ,\uu,\w^l_{\lambda(s+1)},\ldots \w^l_{\lambda d}], \ldots, \w^L_{\lambda},\mathbf{b}_\lambda),y_i)
	\end{split}
\end{equation*}
denote the loss as a function on the parameters into the $s$-th neuron in the $l$-th layer (encoded in the $s$-th row of $\w^l$) after reparameterization. 

For the same layer $l$, we define a linear function $\eta_s:\RR^m\rightarrow\RR^m$ by \[\eta_s(\uu)=\eta_s(u_1,u_2,\ldots, u_{m})=(u_1\lambda^l_{(s,1)},u_2\lambda^l_{(s,2)},\ldots, u_{m}\lambda^l_{(s,m)}).\] By assumption, we have that $\tilde F(\eta_s(\w^l_s))=F(\w^l_s)$ for all $\w^l_s$. By the chain rule, we compute for any coordinate $u_t$ of $\uu$,
\[ \frac{\partial F(\uu)}{\partial u_t}\Bigr|_{\uu=\w^l_s} 
= \frac{\partial \tilde F(\eta_s(\uu))}{\partial u_t}\Bigr|_{\uu=\w^l_s}\]
\[= \sum_k \frac{\partial \tilde F(\eta_s(\uu))}{\partial (\eta_s(\uu)_k)}\Bigr|_{\eta_s(\uu)=\eta_s(\w^l_s)} \cdot \frac{\partial (\eta_s(\uu)_k)}{\partial u_t}\Bigr|_{\eta_s(\uu)=\eta_s(\w^l_s)} \]
\[=  \frac{\partial \tilde F(\vv)}{\partial v_t}\Bigr|_{\vv= \w^l_{\lambda s}}\cdot \lambda^l_{(s,t)}. \]
Similarly, for 
\begin{equation*}
	\begin{split}
		G(\uu, \uu'):=\sum_{i=1}^{|S|} \loss(f(x_i; \w^1,\w^2,&\ldots,[\w^l_1,\ldots,\w^l_{s-1},\uu,\w^l_{s+1},\ldots,\w^l_{s'-1},\uu',\w^l_{s'+1},\ldots  \w^l_{d}],\ldots \\ &\ldots,\w^L,,\mathbf{b}),y_i)
	\end{split}
\end{equation*}
denoting the loss as a function on the parameters of the $s$-th and $s'$-th neuron in the $l$-th layer (encoded in the $s$-th and $s'$-th row of $\w^l$) before reparameterization and for
\begin{equation*}
	\begin{split}
		\tilde G(\uu,\uu'):=&\sum_{i=1}^{|S|} \loss( f(x_i;  \w^1_{\lambda},\w^2_{\lambda} ,\ldots \\ &\ldots ,[\w^l_{\lambda 1} ,\ldots,\w^l_{\lambda(s-1)} ,\uu,\w^l_{\lambda(s+1)},\ldots, \w^l_{s'-1},\uu',\w^l_{s'+1},\ldots \w^l_{\lambda d}], \ldots
		\\ &\ldots \w^L_{\lambda},,\mathbf{b}_\lambda),y_i)
	\end{split}
\end{equation*}
we have $\tilde G(\eta_s(\w_s^l),\eta_{s'}(\w_{s'}^l))=G(\w_s^l,\w_{s'}^l)$. For all $s,s',t,t'$ we obtain second derivatives
%\[ \frac{\partial^2 F(\uu)}{\partial u_t \partial u_r }\Bigr|_{\uu=\w^l_s} = \lambda^l_{(s,t)} \lambda^l_{(s,r)}  \frac{\partial^2 \tilde F(\vv)}{\partial v_t \partial v_r }\Bigr|_{\vv= \w^l_{\lambda s}}. \]
\[ \frac{\partial^2 G(\uu,\uu')}{\partial u_t \partial u'_{t'}}\Bigr|_{\uu=\w^l_s, \uu'=\w^l_{s'}} = \lambda^l_{(s,t)} \lambda^l_{(s',t')}  \frac{\partial^2 \tilde G(\uu,\uu')}{\partial u_t \partial u'_{t'} }\Bigr|_{\uu= \w^l_{\lambda s},\uu'= \w^l_{\lambda s'}}. \]
Consequently, the Hessian $H(\w^l,S)$ of the empirical risk before reparameterization and the Hessian $\tilde H (\w^l_\lambda,S)$ after reparameterization satisfy at the position corresponding to $w_{s,t}$ and $w_{s',t'}$ that
\[H_{s,s'}(\w^l,S)_{(t,t')} =\lambda^l_{(s,t)} \lambda^l_{(s',t')} \cdot \tilde H_{s,s'}(\w^l_\lambda)_{(t,t')}.\] 

Assuming that $\lambda^l_s:=\lambda^l_{(s,t)} =\lambda^l_{(s,t')}$ for all $s,t$ and $t'$, then we get that
%\begin{equation*}
%	\begin{split}
%		\kappa^l(\w) &= \sum_{s=1}^d ||\w^l_{s}||^2 \cdot \lambda_{max}(H_{s,s}(\w^l_s,S)) \\
%			&= \sum_{s=1}^d ||\frac{1}{\lambda^l_s} \w^l_{\lambda s}||^2 \cdot (\lambda^l_{s})^2\cdot \lambda_{max}(\tilde H_{s,s}(\w_{\lambda s}^l,S)) \\
%			&= \sum_{s=1}^d ||\w^l_{\lambda s}||^2 \cdot \lambda_{max}(\tilde H_{s,s}(\w_{\lambda s}^l,S)) \\
%			&= \kappa^l(\w_\lambda) \\
%	\end{split}
%\end{equation*}
%and similarly
\begin{equation*}
	\begin{split}
		\kappa_{Tr}^l(\w) &= \sum_{s,s'=1}^d \langle \w^l_s,\w^l_{s'} \rangle \cdot Tr(H_{s,s'}(\w^l,S)) \\
			&= \sum_{s,s'=1}^d \langle \frac{\w^l_{\lambda s}}{\lambda^l_s},\frac{\w^l_{\lambda s'}}{\lambda^l_{s'}} \rangle \cdot Tr( \lambda^l_{s} \lambda^l_{s'} \tilde H_{s,s'}(\w_\lambda,S)) \\
		&= \sum_{s,s'=1}^d \langle \w^l_{\lambda s}, \w^l_{\lambda s'} \rangle \cdot Tr( \tilde H_{s,s'}(\w_\lambda,S)) \\
		&= \kappa_{Tr}^l(\w_\lambda)\\
	\end{split}
\end{equation*}
This proves Theorem~\ref{thm:invarianceAll}. 

To show the corollary, we first observe that all layer-wise reparameterizations are covered by the theorem. To see this, we only need to check that the condition $\lambda^l_{(s,t)}=\lambda^l_{(s,t')}$ holds for each $s,t$ and $t'$. For layer-wise reparameterizations, we even have that $\lambda^l_{(s,t)}=\lambda^l$ for all $s,t$, since all weights of one layer are multiplied by the same scalar $\lambda^l$, and $\lambda^l_{(s,t)}=\lambda^l_{(s,t')}$ is easily seen to hold true. 

Note further, that any neuron-wise reparameterization given by multiplying all weights into a neuron in a layer $\iota\neq l-1$ by $\lambda>0$ and dividing all outgoing weights by $\lambda$ is also covered by the theorem. Hence, the only neuron-wise reparameterization that can change the relative flatness measures is the one multiplying some row of $\w^{l-1}$ by some $\lambda>0$ and dividing the corresponding column of $\w^l$ by the same $\lambda$. However, by multiplying both $\w^{l-1}$ and $\w^l$ with $V^{-1}$ and $V$ from the left and right respectively, we perform an explicit neuron-wise reparameterization that chooses a unique representative and therefore removes the dependence on such reparamerizations.
\end{proof}

\newpage

\subsection{Proof of Theorem~\ref{thm:averaging}}\label{proof:thm5}

In this section, we prove %a generalized version of 
Theorem~\ref{thm:averaging}. 
%The following theorem extends Theorem~\ref{thm:averaging} by a statement on maximal relative flatness from Definition~\ref{def:flatnessMeasureAppendix}.
%\begin{theorem}\label{thm:averagingAppendix}
%Consider a model $f(\w,x )=\psi(\w,\phi(x))=g(\w \phi(x))$ and a loss function $\loss$. Let $O_m\subset \RR^{m\times m}$ denote the set of orthogonal matrices. Let $\delta$ be a positive (small) real number and $\w=\omega\in\RR^{d\times m}$ denote parameters at a local minimum of the empirical risk on a sample set $S$. If the labels satisfy that 
%$y[\phi_{\delta A}(x_i)]=y[\phi(x_i)]=y_i$
%for all $(x_i,y_i)\in S$, then 
%(i) for each feature selection matrix $||A||\leq 1$ the model $f(x,\omega)$ is $\left ( (\delta, S, A), \epsilon \right )$-feature robust for $\epsilon= \frac{\delta^2 d}2 \kappa^\phi(\omega) +\mathcal{O}(\delta^3)$, and  
%(ii)
% $f(x,\omega)$ is $\left ( (\delta, S, O_m), \epsilon \right )$-feature robust on average over $O_m$  for $\epsilon= \frac{\delta^2}{2m} \kappa^\phi_{Tr}(\omega)+\mathcal{O}(\delta^3)$. 
%\end{theorem}
%\textbf{Note that the statement of the theorem in the main text contains one impression by not specifying that the matrix $A$ used in the condition $y[\phi_{\delta A}(x_i)]=y[\phi(x_i)]=y_i$ refers to an arbitrary orthogonal matrix $A\in\mathcal{O}_m$.} For convenience of the reader, we repeat the assumptions and the statement we aim to show.
For clarity, we repeat the assumptions and the statement we prove in this section:

We consider a model $f(x,\w)=g(\w\phi(x))$, a loss function $\loss$ and a sample set $S$, and let $O_m\subset \RR^{m\times m}$ denote the set of orthogonal matrices. Let $\delta$ be a positive (small) real number and $\w=\omega\in\RR^{d\times m}$ denote parameters at a local minimum of the empirical risk on a sample set $S$. If the output function satisfies that 
$y[\phi_{\delta A}(x_i)]=y[\phi(x_i)]=y_i$
%$y[\phi(x_i+\delta A\phi(x_i))]=y(\phi(x_i))=y_i$
for all $(x_i,y_i)\in S$ and all matrices $||A||\leq 1$, then 
%(i) for each feature selection matrix $||A||\leq 1$ the model $f(x,\omega)$ is $\left ( (\delta, S, A), \epsilon \right )$-feature robust for $\epsilon= \frac{\delta^2 d}2 \kappa^\phi(\omega) +\mathcal{O}(\delta^3)$, and  
%(ii)
we want to show that  $f(x,\omega)$ is $\left ( (\delta, S, O_m), \epsilon \right )$-feature robust on average over $O_m$  for $\epsilon= \frac{\delta^2}{2m} \kappa^\phi_{Tr}(\omega)+\mathcal{O}(\delta^3)$, i.e.,
\[ \left |\Ecal_{\mathcal{F}}^\phi(f,S,\alpha \Acal)\right | \leq \frac{\delta^2}{2m} \kappa^\phi_{Tr}(\omega)+\mathcal{O}(\delta^3) \textrm{ for all }0\leq\alpha\leq\delta\]

\begin{proof}
Writing $z_i=\phi(x_i)$ and $\Ecal_{emp}(\w,S)=\Ecal_{emp}(f(\w,x),S)$ and using the assumption that $y[\phi_{\delta A}(x_i)]=y_i$ for all $(x_i,y_i)\in S$ and all $||A||\leq 1$, we have for any $0\leq \alpha\leq \delta$,
\begin{equation}\label{eq:calculation1}
\begin{split}
\Ecal_\mathcal{F}^\phi(f,S,\alpha A)+\Ecal_{emp}(\w,S) &= \frac 1{|S|} \sum_{i=1}^{|S|} \loss(\psi[\w, \phi_{\alpha A}(x_i)],\ y[\phi_{\alpha A}(x_i)])\\
&=\frac{1}{|S|} \sum_{i=1}^{|S|} \loss(\psi(\w, z_i+\alpha A z_i),y_i) \\&= \frac{1}{|S|} \sum_{i=1}^{|S|} \loss(\psi(\w + \alpha \w A,z_i),y_i)\\
& = \Ecal_{emp}(\w+\alpha \w A,S)
\end{split}
\end{equation}
The latter is the empirical error $\Ecal_{emp}(\w+\alpha \w A,S)$ of the model $f$ on the sample set $S$ at parameters $\w + \alpha \w A$. If $\delta$ is sufficiently small, then by Taylor expansion around the local minimum $\w=\omega$, we have up to order of $\mathcal{O}(\delta^3)$ that 
\begin{equation}\label{eq:calculation2}
\begin{split}
         \Ecal_{emp}(\omega + \alpha  \omega A,S)  &=  \Ecal_{emp}(\omega,S) +  \frac{\alpha^2}{2}  \sum_{s,t=1}^d (\omega_s A) \cdot H_{s,t}( \omega,\phi(S)) \cdot ( \omega_t A)^T  \\
         &\leq \Ecal_{emp}(\omega,S) +  \frac{\delta^2}{2}  \sum_{s,t=1}^d (\omega_s A) \cdot H_{s,t}( \omega,\phi(S)) \cdot ( \omega_t A)^T \\
         \end{split}
\end{equation}
where $\omega_s$ denotes the $s$-th row of $\omega$.

We consider the set of orthogonal matrices $O_m$ as equipped with the (unique) normalized Haar measure. (For the definition of the Haar measure, see e.g. \cite{Krantz}.) We need to show that $\Exp{A\sim O_m}{\Ecal^\phi_{\mathcal{F}}(f,S, \alpha A)}\leq  \frac{\delta^2}{2m}  \sum_{s,t} \langle \w_s,\w_t \rangle  \cdot Tr(H_{s,t})$ for all $0\leq \alpha\leq \delta$ 
with $\Ecal^\phi_{\mathcal{F}}(f,S, A)$ defined as in Eq.~\ref{eq:defRobustness}. Using~\eqref{eq:calculation1} and~\eqref{eq:calculation2} we get 
	\begin{equation*}
		\Exp{A\sim O_m}{\Ecal_\mathcal{F}^\phi(f,S,\alpha A)}\leq \Exp{A\sim O_m}{\frac{\delta^2}{2}  \sum_{s,t=1}^d (\omega_s A)  H_{s,t}(\omega ,S) ( \omega_t A)^T } +\mathcal{O}(\delta^3)
		\end{equation*}

	%If $\w_{i\ast}\neq 0$, then by Proposition~3.2.1 of \cite{Krantz} and the change of variables formula for measures, we get for each $i,j$
	%\begin{equation}\label{eq:Krantz}
	%\Exp{A\sim O_m}{ (\w_{i\ast}A) H\Ecal_{emp}(\w_{j\ast},S) (\w_{i\ast}A)^T} = ||\w_{i\ast}||^2 \Exp{z\in\RR^m, ||z||=1}{z^T H\Ecal_{emp}(\w_{j\ast},S)z}
	%\end{equation}
	%for all $1\leq i,j\leq d$, where the latter expectation is taken over the normalized (uniform) Hausdorff measure over the sphere $S^{m-1}\subset\RR^m$. 
	
	Using the unnormalized trace $Tr([m_{s,t}])=\sum_s m_{s,s}$ we compute with the help of the so-called Hutchinson's trick: 
	%\begin{equation}\label{eq:Hutchinson}
	%\begin{split}
	%\Exp{z\in\RR^m, ||z||=1}{z^T H\Ecal_{emp}(\w_{j\ast},S)z} & = \Exp{||z||=1}{Tr(z^T H\Ecal_{emp}(\w_{j\ast},S)z)} \\
	%& = \Exp{||z||=1}{Tr(H\Ecal_{emp}(\w_{j\ast},S)zz^T)}\\
	%& = Tr(H\Ecal_{emp}(\w_{j\ast},S)\Exp{||z||=1}{ zz^T}).
	%\end{split}
	%\end{equation}
	%Note that $zz^T=[z_iz_j]_{i,j}$ and due to symmetry $\Exp{||z||=1}{z_iz_j}=\Exp{||z||=1}{z_i(-z_j)}$ for $i\neq j$, hence $\Exp{||z||=1}{z_iz_j}=0$ whenever $i\neq j$. Further $\Exp{||z||=1}{z_i^2}=\frac 1m \Exp{||z||=1}{\sum_{i=1}^m z_i^2}=\frac 1m \Exp{||z||=1}{||z||^2}=\frac 1m$ for all $i$. 
 
\begin{equation*}
\begin{split}
Tr(\mathbb{E}_{A\sim O_m}\left [ (\omega_t A)^T (\omega_s A) \right ])&=\mathbb{E}_{A\sim O_m}\left [ Tr((\omega_t A)^T (\omega_s A) \right )] \\
&= \mathbb{E}_{A\sim O_m}\left [ Tr((\omega_s A) (\omega_t A)^T \right )]\\
&= \mathbb{E}_{A\sim O_m}\left [ Tr(\omega_s  \omega_t^T \right )]\\
&=  \langle \omega_s, \omega_t\rangle\\
\end{split}
\end{equation*}
We can interchange two vector coordinates by multiplication of a suitable orthogonal matrix $B$. Since the Haar measure is invariant under multiplication of an orthogonal matrix, the diagonal of $\mathbb{E}_{A\sim O_m}\left [ (\omega_t A)^T (\omega_s A) \right ])$ must contain a constant value. This value along the diagonal must then equal $\frac 1m \langle \omega_s, \omega_t\rangle$. Further, we can multiply one vector coordinate by $(-1)$ via multiplication by an orthogonal matrix, and hence the off-diagonal entries of  $\mathbb{E}_{A\sim O_m}\left [ (\omega_t A)^T (\omega_s A) \right ])$ must be zero, giving that 
\[\mathbb{E}_{A\sim O_m}\left [ (\omega_t A)^T (\omega_s A) \right ])=\frac{\langle \omega_s,\omega_t\rangle }{m}\cdot I.\]
Therefore 
\begin{equation*}
\begin{split}
\mathbb{E}_{A\sim O_m}\left [ (\omega_s A) H_{s,t} (\omega_t A)^T \right ] &= Tr \left ( \mathbb{E}_{A\sim O_m}\left [ (\omega_s A) H_{s,t} (\omega_t A)^T \right ]\right ) \\
&= \mathbb{E}_{A\sim O_m}\left [ Tr ((\omega_s A) H_{s,t} (\omega_t A)^T ) \right ] \\
&= \mathbb{E}_{A\sim O_m}\left [ Tr (H_{s,t} (\omega_t A)^T (\omega_s A)) \right ] \\
&=  Tr (H_{s,t}\cdot \mathbb{E}_{A\sim O_m}\left [ (\omega_t A)^T (\omega_s A)) \right ] \\
&= Tr (H_{s,t} \cdot \frac{\langle \omega_s,\omega_t\rangle }{m}\cdot I)\\
& = \frac{\langle \omega_s,\omega_t\rangle }{m} Tr(H_{s,t}) \\
\end{split}
\end{equation*}
Putting things together, we have for the local optimum $\w=\omega$ that
\begin{equation*}
	\begin{split}
		\Exp{A\sim O_m}{\Ecal_\mathcal{F}^\phi(f, S,\alpha A)} & \leq \frac{\delta^2}{2} \sum_{s,t=1}^d \Exp{A\sim O_m}{ (\omega_s A) H_{s,t}(\omega_t A)^T} + \mathcal{O}( \delta^3 ) \\
		&= \frac{\delta^2}{2m} \sum_{s,t=1}^d \langle \omega_s,\omega_t \rangle \cdot Tr(H_{s,t})+ \mathcal{O}( \delta^3 ) \\
		& = \frac{\delta^2}{2m} \kappa_{Tr}^\phi(\omega) + \mathcal{O}( \delta^3 )\\
		\end{split}
\end{equation*}
\end{proof}

We can further generalize Theorem~\ref{thm:averaging} to more complex labels by introducing  a notion of approximately locally constant labels. 
The following definition frees us from the strong assumption of locally constant labels, i.e. $y[\phi_{\delta A}(x_i)]=y[\phi(x_i)]=y_i$
for all $(x_i,y_i)\in S$ and all matrices $||A||\leq 1$, while still restricting label changes to be one order smaller than the contribution of flatness.

\begin{definition}\label{def:approximateLocallyConstantLabels}
Let $\Dcal$ be a data distribution on a labeled sample space $\Xcal\times\Ycal$ and $S$ a finite iid sample of $\Dcal$. Let $f=\psi\circ \phi$ be a model composed into a feature extractor $\phi$ and predictor $\psi$. We say that $\Dcal$ has approximately locally constant labels of order three around the points $(x,y)\in S$ in feature space $\phi$, if there is some constant $C$ such that
\[\frac{1}{|S|} \sum_{i=1}^{|S|} \Big |\loss(\psi( \phi(x_i)+\Delta_i),y[\phi(x_i)+\Delta_i])- \loss(\psi( \phi(x_i)+\Delta_i),y_i)\Big |\leq C\delta^3 \text{ for }||\Delta_i||\leq \delta||\phi(x_i)||\]
\end{definition}

\begin{corollary}\label{cor:averaginAppendix}
Consider a model $f(x,\w)=\psi(\w,\phi(x))= g(\w\phi(x))$ as above, a loss function $\loss$ and a sample set $S$, and let $O_m\subset \RR^{m\times m}$ denote the set of orthogonal matrices. Let $\delta$ be a positive (small) real number and $\w=\omega\in\RR^{d\times m}$ denote parameters at a local minimum of the empirical risk on a sample set $S$. If $\Dcal$ has approximately locally constant labels of order three around the points $(x,y)\in S$ in feature space, then $f(x,\omega)$ is $\left ( (\delta, S, O_m), \epsilon \right )$-feature robust on average over $O_m$  for $\epsilon= \frac{\delta^2}{2m} \kappa^\phi_{Tr}(\omega)+\mathcal{O}(\delta^3)$. 
\end{corollary}

\begin{proof}
As before, we abbreviate $\phi(x_i)$ by $z_i$. We only need to modify \eqref{eq:calculation1} to account for the strictly weaker assumption on the labels.  For this, we perform Taylor approximation with respect to the labels at $y[\phi(x_i)]=y_i$ to obtain
\begin{equation*}
\begin{split}
\Ecal_\mathcal{F}^\phi(f,S,\alpha A)+\Ecal_{emp}(\w,S) &= \frac 1{|S|} \sum_{i=1}^{|S|} \loss(\psi[ \phi_{\alpha A}(x_i)],\ y[\phi_{\alpha A}(x_i)])\\
&\stackrel{Def~\ref{def:approximateLocallyConstantLabels}}{\leq}\frac{1}{|S|} \sum_{i=1}^{|S|} \loss(\psi(\w, z_i+\alpha A z_i),y_i) +\mathcal{O}(\delta^3)\\
&= \frac{1}{|S|} \sum_{i=1}^{|S|} \loss(\psi(\w + \alpha \w A,z_i),y_i)+\mathcal{O}(\delta^3) \\
&=\Ecal_{emp}(\w+\alpha \w A,S) +\mathcal{O}(\delta^3) \\
%& \stackrel{\eqref{eq:calculation2}}{\leq} \Ecal_{emp}(\omega,S) +  \frac{\delta^2}{2}  \sum_{s,t=1}^d (\omega_s A) \cdot H_{s,t}( \omega,\phi(S)) \cdot ( \omega_t A)^T 
\end{split}
\end{equation*}
The rest of the proof follows the arguments used to show Theorem~\ref{thm:averaging}.

\end{proof}

\newpage
\subsection{Proof of Theorem~\ref{thm:FlatnessGenBound}}
\label{proof:thm6}

To prove Theorem~\ref{thm:FlatnessGenBound}, we will require a proposition that bounds $\epsilon$-representativeness for $\lambda_i$ the local densities from Proposition~\ref{prp:lambdaA}. This is achieved in Proposition~\ref{prop:KDEconstantLabels} below uniformly over all distributions $\Dcal$ that satisfy mild regularity assumptions necessary for a well-defined kernel density estimation.  We first compose the proof to Theorem~\ref{thm:FlatnessGenBound} and subsequently show the arguments leading to the required proposition. 

The main idea to prove Theorem~\ref{thm:FlatnessGenBound} is that the family of distributions considered in Proposition~\ref{prp:lambdaA} (a) provides an explicit link between $\epsilon$-representativeness and feature robustness (Proposition~\ref{prp:lambdaA} and Equation~\ref{eq:splitError}), (b) allows us to approximately bound feature robustness by relative flatness (Theorem~\ref{thm:averaging}), and (c) allows us to apply a kernel density estimation to uniformly bound $\epsilon$-representativeness (Proposition~\ref{prop:KDEconstantLabels}).

Theorem~\ref{thm:FlatnessGenBound} is the informal counterpart to the following version. 
%Part (i) of this theorem provides a bound based on maximal relative flatness (Definition~\ref{def:flatnessMeasureAppendix}) that leaves $\epsilon$-representativness as an unknown term. 

\begin{theorem}\label{thm:FlatnessGenBound_precise}
Consider a model $f(x,\w)=g(\w\phi(x))$, a loss function $\loss$, a sample set $S$, and let $m$ denote the dimension of the feature space defined by $\phi$ and let $\delta$ be a positive (small) real number. Let $\omega\in\RR^{d\times m}$ denote a local minimum of the empirical risk on an iid sample set $S$. 

Suppose that the distribution $\Dcal$ has a smooth density $p^\phi_\Dcal$ on the feature space $\RR^m$ such that $\int_z  \left | \nabla^2 \left( p^\phi_\Dcal(z)||z||^2\right ) \right | dz $ and $\int_z \frac{p^\phi_\Dcal(z)}{||z||^m}\ dz$ are well-defined and finite. Then for sufficiently large sample size $|S|$, if the distribution has approximately locally constant labels of order three (see Definition~\ref{def:approximateLocallyConstantLabels}), then it holds with probability $1-\Delta$ over sample sets $S$ that
\[\Ecal_{gen}(f(\cdot, \omega),S)\lesssim |S|^{-\frac{2}{4+m}} \left ( \frac{\kappa^\phi_{Tr}(\omega)}{2m}  +  C_1(p^\phi_\Dcal,L)+\frac{C_2(p^\phi_\Dcal,L)}{\sqrt{\Delta}} \right ) \] up to higher orders in $|S|^{-1}$ for constants $C_1,C_2$ that depend only on the distribution in feature space $p^\phi_\Dcal$ induced by $\phi$ and the chosen $|S|$-tuple $\Lambda_{\delta}$ as in Proposition~\ref{prp:lambdaA} and the maximal loss $L$.
\end{theorem}

\begin{proof}
The proof combines Equation~\ref{eq:splitError} with Proposition~\ref{prop:KDEconstantLabels} and Theorem~\ref{thm:averaging}. At first we use  Equation~\ref{eq:splitError} to split the generalization gap into $\Ecal_{gen}(f) = \Ecal_{Rep}^\phi(f,S,\Lambda_{\mathcal{A}_\delta}) + \Ecal_{\mathcal{F}}(f,S, \Acal_\delta)$. For the family of distributions $\Lambda_\delta$ from Proposition~\ref{prp:lambdaA}, we have by Proposition~\ref{prop:KDEconstantLabels} that 
\[ | \Ecal_{Rep}^\phi(\psi\circ \phi,S,\Lambda^\phi_\delta) | \leq \left (C_1(p^\phi_\Dcal,L)+\frac{C_2(p^\phi_\Dcal,L)}{\sqrt{\Delta}}\right ) \cdot  |S|^{-\frac{2}{4+m}}  + \mathcal{O}(|S|^{-\frac{3}{4+m}})\] 
when $\delta=|S|^{-\frac{1}{4+m}}$.  With $\Ecal_{\mathcal{F}}(f,S,\Acal_\delta) = \Exp{A\sim \mathcal{A}_\delta}{ \Ecal_{\mathcal{F}}(f,S, A)}$, we use (the proof to) Proposition~\ref{prp:lambdaA} to see that this can be written as 
\[\Exp{A\sim \mathcal{A}_\delta}{ \Ecal_{\mathcal{F}} (f,S, A)}= \Exp{0\leq \alpha \leq \delta}{\Exp{A\sim (\mathcal{O}_m,\mu)}{ \Ecal_{\mathcal{F}} (f,S, \alpha A)}}\] 
where $O_m\subset \RR^{m\times m}$ denote the set of orthogonal matrices and $\mu$ the Haar measure on this set. %For sufficiently large $|S|$, the local minimum guarantees that \[\Exp{0\leq d \leq \delta}{\Exp{A\sim (\mathcal{O}_m,\mu)}{ \Ecal_{\mathcal{F}} (f,S, A)}}\leq \Exp{A\sim (\mathcal{O}_m,\mu)}{ \Ecal_{\mathcal{F}} (f,S, \delta A)}.\] 
Finally, Theorem~\ref{thm:averaging} bounds the latter by $\frac{|S|^{-\frac{2}{4+m}}}{2m}\kappa^\phi_{Tr}(\omega) $ up to higher orders in $|S|^{-1}$.\\
\end{proof}
% \clearpage

We finally prove that the bound on $\epsilon$-representativeness in the proof to the preceding Theorem indeed holds true.

\begin{proposition}\label{prop:KDEconstantLabels}
Consider a model $f(x,\w)=\psi(\w,\phi(x))$, a loss function $\loss$ and let $S\subseteq \Xcal\times \Ycal$ be a finite sample set. With $x_i\in S$, let $\lambda_i(z)=k_{\delta||\phi(x_i)||}(0,z)$ define an $|S|$-tuple $\Lambda_\delta$ of densities as in Proposition~\ref{prp:lambdaA} and assume that the loss function is bounded by $L$. Suppose that the distribution $\Dcal$ has a smooth density $p^\phi_\Dcal$ on a feature space $\RR^m$ such that $\int_z  \nabla^2 \left( p^\phi_\Dcal(z)||z||^2\right )  dz $ and $\int_z \frac{p^\phi_\Dcal(z)}{||z||^m}\ dz$ are well-defined and finite.  Then there exist constants $C_1(p^\phi_\Dcal,L),C_2(p^\phi_\Dcal,L)$ depending on the distribution and the maximal loss such that, with probability $1-\Delta$ over possible sample sets $S$, $\epsilon$-interpolation is bounded for $\delta=|S|^{-\frac{1}{4+m}}$ by
\[ | \Ecal_{Rep}^\phi(f,S,\Lambda_\delta) | \leq \left (C_1(p^\phi_\Dcal,L)+\frac{C_2(p^\phi_\Dcal,L)}{\sqrt{\Delta}}\right ) \cdot  |S|^{-\frac{2}{4+m}}  + \mathcal{O}(|S|^{-\frac{3}{4+m}}) \]
\end{proposition}

\begin{proof}
We let
\[\hat p(z) = \frac{1}{|S|} \sum_{i=1}^{|S|}  k_{\delta||\phi(x_i)||}(\phi(x_i),z)\] 
%The assumption of locally constant labels allows us  to write 
%\[\hat p(z)\cdot  \loss(\psi(z),y(z)) = \frac{1}{|S|} \sum_{i=1}^{|S|}  k_{\delta||\phi(x_i)||}(\phi(x_i),z) \cdot \loss(\psi(z),y_i)\]
With $\lambda_i=k_{\delta||\phi(x_i)||}(0,z)$ 
%the densities defined for each $i$ by density $k_{\delta||\phi(x_i)||}(0,\xi)$ as in \eqref{eq:kernel},
we have
\begin{equation}
\begin{split}
\left |\vphantom{\Ecal_{Rep}^\phi (f,S,\Lambda_\delta)} \right.&\left. \Ecal_{Rep}^\phi (f,S,\Lambda_\delta) \right | 
=\left |  \Ecal(f)   -\frac 1{|S|} \sum_{i=1}^{|S|} \mathbb{E}_{\xi \sim \lambda_i}\left[\loss( \psi(\phi(x_i)+ \xi),y[\phi(x_i)+ \xi]\right]  \right |  \\
&=\left |  \int_z p^\phi_{\Dcal}(z) \cdot \ell(\psi(z),y(z))\ dz - \frac 1{|S|} \sum_{i=1}^{|S|} \int_z k_{\delta||\phi(x_i)||}(\phi(x_i),z) \cdot \ell(\psi(z),y(z))\ dz \right |  \\
&\leq \underbrace{ \left |  \int_z (p^\phi_{\Dcal}(z) - \mathbb{E}_S\left[ \hat p(z) \right]) \cdot \ell(\psi(z),y(z))\ dz \right |}_{(I)} + \underbrace{\left |  \int_z \left (\mathbb{E}_S\left[ \hat p(z) \right] -\hat p(z) \right )\cdot \loss(\psi(z),y(z)) \ dz \right  |}_{(II)}   \\
%&=\left |  \int_x p_\Dcal(x) \ell(x,y(x))\ dx - \frac 1n \sum_{i=1}^n \int_x k_\delta(x_i,x) \ell(f(x),y(x))\ dx \right |  \\
%& \leq  L  \int_x \left | p_\Dcal(x) - \frac 1n \sum_{i=1}^n k_\delta(x_i,x) \right | dx
\end{split}
\label{eq:intAsVKDE}
\end{equation}
For the further analysis, we make use of \citet{jones1994variable} and combine it with the generalization to the multivariate case in Chp. 4.3.1 in~\citet{silverman1986density}. A Taylor approximation with respect to the bandwidth of the kernel $\delta$ yields
\[ (I) =  \frac{\delta^2}{2}\tau_2 \left | \int_z  \nabla^2 \left( p^\phi_\Dcal(z)||z||^2\right )  \loss(\psi(z),y(z)) dz \right | +\mathcal{O}(\delta^3) \]
where 
\[
\tau_2 = \int_{z} \|z\|^2 k_1(0,z) dz.
\]
For (II) we consider the random variable $Z=\int_z \hat p(z) \loss(\psi(z),y(z))\ dz$ as a function on the set of possible sample sets of a fixed size. Applying Chebychef's inequality on $Z$, we get that 
\[
Pr\left(\left|Z-\Exp{S}{Z}\right|> \epsilon_{est}\right)\leq \frac{Var(Z)}{\epsilon_{est}^2} =: \Delta\enspace .
\]
Solving for $\epsilon_{est}$ yields that with probability $1-\Delta$ we have
\[(II) = |Z-\Exp{S}{Z}| \leq \frac{\sqrt{Var(Z)}}{\sqrt{\Delta}} \] 
Further, the variance of $Z$ can be bounded by 
\[Var(Z) = \Exp{S}{(Z- \Exp{S}{Z})^2}\]
\[ = \Exp{S}{\left (\int \hat p(z) \loss(\psi(z),y(z))\ dz - \Exp{S}{\int \hat p(z)\loss(\psi(z),y(z))\ dz} \right )^2} \]
\[ = \Exp{S}{\left (\int  \left ( \hat p(z) - \Exp{S}{ \hat p(z)}\right )  \loss(\psi(z),y(z)) \ dz \right )^2} \]
\[ \leq \underbrace{  \Exp{S}{\int  \left ( \hat p(z) - \Exp{S}{\int \hat p(z)}\right )^2\  dz}}_{(III)} \cdot \underbrace{\left ( \int_z \loss(\psi(z),y(z))^2\ dz\right )}_{\leq L^2 \text{Vol}(\phi(\Dcal))}  \]
It follows from Eq. (2.3) in \citet{jones1994variable} together with Eq. 4.10 in~\citet{silverman1986density} for (III) that for small $\delta$ and large sample size $|S|$ the term (III), i.e., the variance of $\tilde{p}$, is given by 
\[
(III) = \beta |S|^{-1}\delta^{-m} \alpha + \mathcal{O}(|S|^{-2})\enspace ,
\]
where $\alpha=\int_z \frac{p^\phi_\Dcal(z)}{||z||^m}\ dz $ and $\beta = \int_z k_1(0,z)^2 \ dz$.
Putting things together gives
\begin{equation*}
\begin{split}
 \Ecal_{Rep}^\phi(f, S,\Lambda_\delta) | \leq & L \frac{\delta^2}{2}\tau_2 \left | \int_z  \nabla^2 \left( p^\phi_\Dcal(z)||z||^2\right ) \right | dz  + \frac{L \sqrt{\alpha \beta }}{\sqrt{\Delta}} \sqrt{\text{Vol}(\phi(\Dcal))}  |S|^{-\frac 12} \delta^{-\frac {m}2}\\ &  + \mathcal{O}(|S|^{-2}) +\mathcal{O}(\delta^3)\enspace .
\end{split}
\end{equation*} 
Choosing the bandwidth as $\delta=|S|^{-\frac{1}{4+m}}$ gives
\begin{equation*}
\begin{split}
| \Ecal_{Rep}^\phi(f, S,\Lambda_\delta) | \leq & |S|^{-\frac{2}{4+m}}\left (  \tau_2 L \left | \int_z   \nabla^2 \left( p^\phi_\Dcal(z)||z||^2\right )\right | dz  + \frac{\sqrt{\alpha \beta } L }{\sqrt{\Delta}}  \sqrt{\text{Vol}(\phi(\Dcal))} \right )\\
& + \mathcal{O}(|S|^{-\frac{3}{m+4}})\enspace .
\end{split}
\end{equation*}
The result follows from setting
\begin{equation*}
\begin{split}
C_1=&\tau_2 L \left | \int_z   \nabla^2 \left( p^\phi_\Dcal(z)||z||^2\right )\right | dz\\
C_2=&{\sqrt{\alpha \beta } L }  \sqrt{\text{Vol}(\phi(\Dcal))}\enspace .\\
\end{split}
\end{equation*}
\end{proof}

\newpage

\section{Relative flatness for a uniform bound over general distributions on feature matrices}\label{appdx:maxFlatness}
This article based its consideration on the specific distribution on feature matrices of Proposition~\ref{prp:lambdaA}, since this distribution allows to use standard results of kernel density estimation in the proof to Theorem~\ref{thm:FlatnessGenBound}. However, the decomposition of the risk in Equation~\ref{eq:splitError} holds for any distribution on feature matrices $\mathcal{A}$ and induced distributions on feature space $\Lambda_{\mathcal{A}}$. To allow maximal flexibility in the choice of a distribution  $\mathcal{A}$ on feature matrices of norm $||A||\leq 1$, we define another version of relative flatness based on the maximal eigenvalues of partial Hessians instead of the trace.

\begin{definition}\label{def:flatnessMeasureAppendix}
For a model $f(\w,x )=g(\w\phi(x))$ with a twice differentiable function $g$, a twice differentiable loss function $\loss$ and a sample set $S$ we define maximal relative flatness by
 \begin{equation}%\label{eq:defFlatness}
 \kappa^\phi(\w) := \sum_{s=1}^d || \w_s||^2  \cdot \lambda_{max}(H_{s,s}(\w,\phi(S)))
 %\kappa^\phi_{Tr}(\w) := \sum_{s,s'=1}^d \langle \w_s,\w_s'\rangle  \cdot Tr(H_{s,s'}(\w,\phi(S))),
\end{equation}
where $\lambda_{max}$ denotes the maximal eigenvalue of a matrix and $H_{s,s'}$ the Hessian matrix as in \eqref{eq:hessians}.
\end{definition}

The analogue to Theorem~\ref{thm:averaging} for maximal relative flatness shows that maximal flatness bounds feature robustness uniformly over all feature matrices of norm $||A||\leq 1$. 

\begin{theorem}\label{thm:averagingAppendix}
Consider a model $f(x,\w)=g(\w\phi(x))$ as above, a loss function $\loss$ and a sample set $S$, and let $O_m\subset \RR^{m\times m}$ denote the set of orthogonal matrices. Let $\delta$ be a positive (small) real number and $\w=\omega\in\RR^{d\times m}$ denote parameters at a local minimum of the empirical risk on a sample set $S$. If the labels satisfy that 
$y[\phi_{\delta A}(x_i)]=y[\phi(x_i)]=y_i$
%$y[\phi(x_i+\delta A\phi(x_i))]=y(\phi(x_i))=y_i$
for all $(x_i,y_i)\in S$ and all $||A||\leq 1$, then, for each feature selection matrix $||A||\leq 1$ the model $f(x,\omega)$ is $\left ( (\delta, S, A), \epsilon \right )$-feature robust for $\epsilon= \frac{\delta^2 d}2 \kappa^\phi(\omega) +\mathcal{O}(\delta^3)$
\end{theorem}

\begin{proof}
Writing $z_i=\phi(x_i)$ and $\Ecal_{emp}(\w,S)=\Ecal_{emp}(f(\w,x),S)$ and using the assumption that $y[\phi_{\delta A}(x_i)]=y_i$ for all $(x_i,y_i)\in S$ and all $||A||\leq 1$, we have by the first part of the proof of Theorem~\ref{thm:averaging} that for any $0\leq \alpha\leq \delta$,
\begin{equation}\label{eq:calculation1b}
\Ecal_\mathcal{F}^\phi(f,S,\alpha A)+\Ecal_{emp}(\w,S)  = \Ecal_{emp}(\w+\alpha \w A,S)
\end{equation}
and 
\begin{equation}\label{eq:calculation2b}
         \Ecal_{emp}(\omega + \alpha  \omega A,S)  \leq \Ecal_{emp}(\omega,S) +  \frac{\delta^2}{2}  \sum_{s,t=1}^d (\omega_s A) \cdot H_{s,t}( \omega,\phi(S)) \cdot ( \omega_t A)^T + \mathcal{O}(\delta^3)
\end{equation}
at a local minimum $\omega$, where $\omega_s$ denotes the $s$-th row of $\omega$.

Note that for $||A||\leq 1$ and a row vectors $\w_s$ it holds that $||\w_s A||\leq ||\w_s||$. Further, since the full Hessian matrix $H(\omega,S)=(H_{s,t}(\omega,S))_{s,t}$ is a positive semidefinite matrix at a local minimum $\omega$, it holds for each row vectors $\w_s,\w_t$ that \begin{equation}\label{eq:posMatrixInequality} 
\w_s H_{s,t}(\omega,S) \w_t^T\leq \frac 12 \Big (\w_s H_{s,s}(\omega,S) \w_s^T+ \w_t H_{t,t}(\omega,S) \w_t^T\Big ),
\end{equation} 
We therefore get that for any feature matrix $A$ with $||A||\leq 1$,
\begin{equation}\label{eq:bound}
	\begin{split}
		 \Ecal^\phi_\mathcal{F}(f,S,\delta A)& \leq \max_{||A||\leq 1} \Ecal_{\mathcal{F}}(f,S,\delta A)\\
		 & \stackrel{\eqref{eq:calculation1b}, \eqref{eq:calculation2b}}{\leq} \max_{||A||\leq 1} \frac{\delta^2}{2}  \sum_{s,t=1}^d (\omega_s A) \cdot H_{s,t}(\omega,S) \cdot ( \omega_t A)^T +  \mathcal{O}(\delta^3)\\\
		&\stackrel{\eqref{eq:posMatrixInequality}}{\leq} \max_{||A||\leq 1} \frac{\delta^2 d}2\sum_{s=1}^d (\omega_s A) \cdot H_{s,s}(\omega,S) \cdot ( \omega_s A)^T+  \mathcal{O}(\delta^3)\\ 		
		& \leq \frac{\delta^2 d}2 \sum_{s=1}^d \max_{||\mathbf{z}||\leq ||\omega_s||}    \mathbf{z}  H_{s,s}( \omega,S) \mathbf{z}^T + \mathcal{O}(\delta^3)\\
		&=\frac{\delta^2 d}2 \sum_{s=1}^d  \max_{||\mathbf{z}|| = 1}||\omega_s||^2\  \mathbf{z}  H_{s,s}( \omega,S)  \mathbf{z}^T+ \mathcal{O}(\delta^3) \\
        &=\frac{\delta^2 d}2  \sum_{s=1}^d ||\omega_s||^2\  \lambda_{max} (H_{s,s}(\omega,S)) + \mathcal{O}(\delta^3)\\
		&=\frac{\delta^2 d}2   \kappa^\phi(\omega)+ \mathcal{O}(\delta^3)\\
	\end{split}    
\end{equation}
where we used the identity that $\max_{||x||=1}x^TMx=\lambda_{max}(M)$ for any symmetric matrix $M$.

\end{proof}

With this, the analogue to Theorem~\ref{thm:FlatnessGenBound} (or its version Theorem~\ref{thm:FlatnessGenBound_precise} in the appendix) allows maximal flexibility to choose $\mathcal{A}_\delta$ (and $\delta>0$) to bound representativeness. This leads to the following generalization bound.

\begin{theorem}
Consider a model $f(x,\w)=g(\w\phi(x))$, a loss function $\loss$, a sample set $S$, and let $m$ denote the dimension of the feature space defined by $\phi$ and let $\delta$ be a positive (small) real number. Let $\omega\in\RR^{d\times m}$ denote a local minimum of the empirical risk on an iid sample set $S$. 

Let $\Upsilon_\delta$ be the set of all  $|S|$-tuple of distributions $\Lambda_{\mathcal{A}_\delta}$ on feature vectors induced by a distribution $\mathcal{A}_\delta$ on feature matrices of norm smaller than $\delta$ as in Section~\ref{sct:inducingDistributions}. Then it holds that  
\[\Ecal_{gen}(f(\cdot,\omega),S) \leq \inf_{ \mathcal{A}_\delta \in \Upsilon_\delta} \Ecal_{Rep}^\phi(f,S,\Lambda_{\mathcal{A}_\delta}) + \frac{\delta^2d}2  \sum_{s=1}^d \kappa^\phi(\omega) + \mathcal{O}(\delta^3).\] 
\end{theorem}

\begin{proof}
Part~(i) follows from combining Equation~\ref{eq:splitError} with Theorem~\ref{thm:averagingAppendix}. First, we use \eqref{eq:splitError} to split the generalization gap into $\Ecal_{gen}(f) = \Ecal_{Rep}^\phi(f,S,\Lambda_{ \mathcal{A}_\delta}) + \Ecal_{\mathcal{F}}(f,S, \Acal_\delta)$. Then, Theorem~\ref{thm:averagingAppendix} shows that  $\Ecal_{\mathcal{F}}(f,S, \Acal_\delta)\leq  \frac{\delta^2d}2  \kappa^\phi(\omega)+ \mathcal{O}(\delta^3)$ as $\Ecal_{\mathcal{F}}(f,S, \delta A)\leq  \frac{\delta^2d}2  \kappa^\phi(\omega)+ \mathcal{O}(\delta^3)$ for all $||A||\leq 1$.
\end{proof}

% \bibliographystyle{plainnat}
% \bibliography{bibliography}
% \end{document}
%\newpage
%\input{sec_checklist}
\end{document}